\documentclass[12pt]{article}

\usepackage{amsmath, amsthm, amssymb}

\usepackage{clrscode}
\usepackage[all]{xy}
\usepackage[round]{natbib}
\usepackage{lscape}       % to print tables in  landscape
\usepackage{booktabs}  % nicer table borders 

\usepackage{fullpage}

\usepackage{graphicx,url}
\usepackage[latin1]{inputenc}  % allows the portuguese letter signs
\usepackage{enumitem}  % allows enumerates with roman numbers
\usepackage{subfigure}
\usepackage{setspace}  % allows singlespacing command

\newtheorem{thm}{Theorem}[section] % Theorem
\newtheorem{prop}[thm]{Proposition} % 
\newtheorem{prob}[thm]{Problem} %
\graphicspath{{./figures/}} 

\begin{document}

\begin{center}
  {\large \bf The U-curve optimization problem: improvements on the original algorithm and time complexity analysis}\\
  \bigskip
  Marcelo S. Reis$^{1}$, Carlos E. Ferreira$^{2}$, and Junior Barrera$^{2}$\\
  \bigskip
  $^1$ LETA-CeTICS, Instituto Butantan, S\~ao Paulo, Brazil.\\
  $^2$ Instituto de Matem\'atica e Estat\'istica, Universidade de S\~ao Paulo, S\~ao Paulo, Brazil.
  \bigskip
\end{center}

\bigskip

\begin{abstract}
The U-curve optimization problem is characterized by a
decomposable in U-shaped curves cost function over the chains of a
Boolean lattice. This problem can be applied to model the classical
feature selection problem in Machine Learning. Recently, the
$\proc{U-Curve}$ algorithm was proposed to give optimal solutions to
the U-curve problem. In this article, we point out that the
$\proc{U-Curve}$ algorithm is in fact suboptimal, and introduce the
$\proc{U-Curve-Search}$ (UCS) algorithm, which is actually
optimal. We also present the results of optimal and suboptimal 
experiments, in which $\proc{UCS}$ is compared with the $\proc{UBB}$
optimal branch-and-bound algorithm and the $\proc{SFFS}$ heuristic,
respectively. We show that, in both experiments, $\proc{UCS}$ had a
better performance than its competitor. Finally, we analyze the obtained
results and point out improvements on $\proc{UCS}$ that might enhance
the performance of this algorithm.
\end{abstract}


\section{Introduction} \label{sec:introduction}

The feature selection problem is a problem in Machine Learning that
has been studied for many years~\citep{Theodoridis:Book}.  It is a
combinatorial optimization problem that consists in finding a subset
$X$ of a finite set of features $S$ that minimizes a cost function
$c$, defined from $\mathcal{P}(S)$,  the power set of $S$, to
$\mathbb{R}+$. If $X$ is a subset of $S$ and $c(X)$ is a minimum
(i.e., there does not exist another subset  $Y$ of $S$ such that $c(Y)
< c(X)$), then $X$ is an optimal set of features or a set of features
of minimum cost. The model of the feature selection problem is defined
by a cost function that has impact in the complexity of the
corresponding search algorithms. In realistic models, the cost
function depends on the estimation of the joint probability
distribution. The feature selection problem is NP-hard~\citep{Meiri:2006}, which
implies that the search for a subset $X$ of $S$ of minimum cost maybe unfeasible (i.e., it may require a computing time that is exponential in the cardinality of $S$). However, for some families of cost functions, there is a guarantee that the problem is not NP-hard  and, consequently, the search problem is feasible; for instance, when the cost function is submodular~\citep{Schrijver:2000}. Moreover, depending on the definition of the cost function, different search algorithms may be employed to solve the problem. Therefore, a feature selection problem relies on two important steps:
\begin{itemize}
\item the choice of an appropriate cost function, also known as feature selection subset criterion, which may lead to either a feasible or an unfeasible search problem;
\item the choice of an appropriate feature subset selection search algorithm for the defined search problem. The chosen search algorithm may be either optimal (i.e., it always finds a subset of minimum cost) or suboptimal (i.e., it may not find a subset of minimum cost).
\end{itemize}

% \subsection{Feature subset selection criteria}

The central problem in Machine Learning is the design of
classifiers from a given joint distribution of a vector of features
and the corresponding labels. Usually, such joint distribution is
estimated from a limited number of samples and the estimation error
depends on the choice of the feature selection vector. A feature
subset selection criterion is a measure of the estimated joint
distribution that intends to identify the availability of small error
classifiers. Some noteworthy feature subset selection criteria are the
Kullback-Leibler distance, the Chernoff
bound~\citep{Narendra:1977}, the divergence distance, the Bhattacharyya distance~\citep{Piramuthu:2004}, and the mean
conditional entropy~\citep{Lin:1991}. All those criteria are measures
of the estimation of joint distribution of a feature subset~\citep{Barrera:2007, Ris:2010}, therefore they are susceptible to the
U-curve phenomenon; that is, for a fixed number of samples, the increase in the number of considered features may have two consequences: if the available sample is enough to a good estimation, then it should occur a reduction of the estimation error, otherwise, the lack of data induces an increase of the estimation error. In practice, the available data are not enough for good estimations. A criterion that is susceptible to the U-curve phenomenon has U-shaped graphs when the function domain is restricted to a chain of the Boolean lattice $(\mathcal{P}(S), \subseteq)$. Such criterion may be used as the cost function of a {\bf U-curve problem}, which models a feature selection problem.

% \subsection{Feature subset selection algorithms}

A feature subset selection algorithm is a search algorithm on the
space $\mathcal{P}(S)$ under the cost function $c$. Some mechanisms of
a general search algorithm are: an order to visit the elements of
$\mathcal{P}(S)$; visit and measure an element $X$ of
$\mathcal{P}(S)$; compare $c(X)$ with the current minimum and update
the minimum if necessary; prune every known element $Y$ of
$\mathcal{P}(S)$ such that $c(Y) > c(X)$; stop when the current
minimum satisfies a given condition. Once it is unknown a polynomial-time
optimal algorithm for the feature subset selection problem, some heuristics that produce suboptimal
solutions are applied~\citep{Pudil:1994, Meiri:2006, Unler:2010}. One of the first introduced heuristics was the
sequential selection: \citet{Marill:1963} presented a feature
selection algorithm for instances whose cost functions were determined
by a divergence distance. This greedy algorithm,
named Sequential Backward Selection ($\proc{SBS}$), is feasible,
though it has the so called ``nesting effect''~\citep{Pudil:1994}. \citet{Whitney:1971} introduced the dual of the $\proc{SBS}$,
the Sequential Forward Selection ($\proc{SFS}$)  algorithm, which also has the nesting effect. Some
heuristics were proposed in order to deal with the nesting
effect~\citep{Kittler:1978}, among them
the Sequential Forward Floating Selection ($\proc{SFFS}$) introduced by~\citet{Pudil:1994}. In the following
years, several works were published suggesting improvements on the
$\proc{SFFS}$ algorithm~\citep{Somol:1999, Nakariyakul:2009}. Nevertheless, the probability that a sequential selection based
algorithm gives an optimal solution decreases with the increase of the
number of features of the problem~\citep{Kudo:2000}, which encourages
the investigation of algorithms that give an optimal solution, a much less explored family of feature selection
algorithms ~\citep{Jain:2000}. \citet{Narendra:1977} presented a
branch-and-bound algorithm for the feature selection problem with cost
functions that are increasing. As in the exhaustive search, the
algorithm always gives an optimal solution, although it may be
unfeasible for a big number of features~\citep{Pudil:1994, Ris:2010}. \citet{Ris:2010} proposed a search algorithm to solve
the U-curve problem, the $\proc{U-Curve}$ algorithm.

% \subsection{Objectives and organization of this article}

The remainder of this article is organized as follows: in Section~\ref{sec:U-curve!algorithm}, we formalize the U-curve problem, point out an error in the $\proc{U-Curve}$ algorithm introduced by Ris et al, and present a modification to fix it. This correction implies in the creation of a new optimal-search algorithm for the U-curve problem, the $\proc{U-Curve-Search}$ ($\proc{UCS}$) algorithm. In Section~\ref{sec:results}, we show some experimental results with $\proc{UCS}$, the $\proc{UBB}$ optimal branch-and-bound algorithm, and the $\proc{SFFS}$ heuristic, and make a comparative analysis under optimal and suboptimal conditions. In Section~\ref{sec:discussion}, we present some analyses on the theoretical and experimental results. Finally, in Section~\ref{sec:conclusion}, we point out the contributions of this paper and suggest some future works for this research.

% End of introduction
%

\section{The $\proc{U-Curve}$ and the $\proc{U-Curve-Search}$
  algorithms} \label{sec:U-curve!algorithm}

Let $S$ be a finite non-empty set. The {\bf power set}
$\mathcal{P}(S)$ is the collection of all subsets of $S$, including
the empty set and $S$ itself. A {\bf chain}  is a collection $\{X_1,
X_2, \ldots, X_k\} \subseteq \mathcal{P}(S)$, such that $X_1 \subseteq
X_2 \subseteq \ldots \subseteq X_k$.  A chain of $\mathcal{P}(S)$ is
{\bf maximal} if it is not properly included by any other chain of
$\mathcal{P}(S)$. Let $\mathcal{X} \subseteq \mathcal{P}(S)$ be a
chain and $f$ be a function that takes values from $\mathcal{X}$ to
$\mathbb{R}+$. $f$ describes a U-shaped curve if, for any $X_1, X_2,
X_3 \in \mathcal{X}$, then $X_1 \subseteq X_2 \subseteq X_3$ implies
that $f(X_2) \le max\{f(X_1), f(X_3)\}$.  A {\bf cost function} $c$ is
a function  defined from $\mathcal{P}(S)$ to $\mathbb{R}+$. $c$ is
{\bf decomposable in U-shaped curves}  if,  for each chain
$\mathcal{X} \subseteq  \mathcal{P}(S)$, the restriction of $c$ to
$\mathcal{X}$  describes a U-shaped curve. Let $X$ be an element of
$\mathcal{X}  \subseteq \mathcal{P}(S)$. If there does not exist
another element  $Y$ of $\mathcal{X}$ such that $c(Y) < c(X)$, then
$X$ is of  {\bf minimum cost} in $\mathcal{X}$. If $\mathcal{X} =
\mathcal{P}(S)$,  then we just say that $X$ is of minimum cost.

Now let us define the central problem that is studied in this article.

\begin{prob}{\em (U-curve problem)}\label{problem_ucurve}
Given a non-empty set $S$ and a decomposable in U-shaped curves cost function $c$, find an element of $\mathcal{P}(S)$ of minimum cost.
\end{prob}

\subsection{The $\proc{U-Curve}$ algorithm}

Let $S$ be a finite non-empty set. An {\bf interval} $[A, B]$ of $\mathcal{P}(S)$ is
defined as $[A,B] := \{ X \in \mathcal{P}(S) : A \subseteq X \subseteq
B\}$. Let  $[\emptyset,  L]$ be an interval with the leftmost term
being the empty set. $L$  is a {\bf lower restriction}. In the same
way, let $[U, S]$ be an interval with the rightmost term being the
complete set. $U$ is an {\bf upper restriction}. An element of
$\mathcal{P}(S)$ may be  represented by the characteristic vector in
$\{0, 1\}^{|S|}$.  In the figures of this article, every Boolean
lattice has its elements described through their characteristic vectors.

Aiming to solve the U-curve optimization problem (problem
\ref{problem_ucurve}), \citet{Ris:2010} presented the $\proc{U-Curve}$ algorithm. The $\proc{U-Curve}$ is an optimal algorithm that has the following properties:
\begin{enumerate}[label=(\roman{*}), ref=(\roman{*})]
\item \label{item:ucurve_one} it works on a search space (the collection of subsets of a finite, non-empty set) that is a Boolean lattice;
\item \label{item:ucurve_two} and optimizes cost functions that are decomposable in U-shaped curves.
\end{enumerate}

The way the $\proc{U-Curve}$ algorithm visits the search space takes into account the properties \ref{item:ucurve_one} and \ref{item:ucurve_two}: at each iteration, the algorithm explores a maximal chain, either in bottom-up or top-down way, until it reaches an element $M$ of minimum cost. Then, it executes a procedure, called ``minimum exhausting'', that performs a depth-first search for elements of cost less or equal to $c(M)$. At the end of the iteration, a list of elements of minimum cost visited so far is updated and the search space is reduced through the update of two collections of elements, called ``upper restrictions'' and ``lower restrictions''. Each element $U$ of the ``upper restrictions'' removes from the search space the elements in the interval $[U, S]$, while each element $L$ of the ``lower restrictions'' removes from the search space the elements in the interval $[\emptyset, L]$. The algorithm performs a succession of iterations until the search space is empty. For a detailed explanation of the algorithm dynamics, refer to \citet{Ris:2010}, particularly to Figure 4.

The $\proc{U-Curve}$ algorithm showed a better performance than the $\proc{SFFS}$ heuristic, that is, for a given instance of the U-curve problem, $\proc{U-Curve}$ finds subsets of cost less or equal to the best subset found by $\proc{SFFS}$, generally with fewer access to the cost function ~\citep{Ris:2010}. To avoid unnecessary access to cost function usually is important in feature selection problems, since often one access is expensive and the search space is very large.

\subsection{Problem with the $\proc{U-Curve}$ algorithm}
\label{sec:UCS:error}

Although the $\proc{U-Curve}$ algorithm was designed to be optimal, an analysis of its correctness resulted in the discovery of a
problem that leads to suboptimal results in some kinds of
instances. This problem is located in the $\proc{Minimum-Exhausting}$
subroutine, which is described in \citet{Ris:2010},
Algorithm 3. $\proc{Minimum-Exhausting}$ performs a depth-first search
(DFS) on the Boolean lattice. The subroutine begins by adding into a
stack an element that has minimum cost in a given chain of $\mathcal{P}(S)$. Let us call $T$ the top of this stack. At the beginning of each iteration, all elements adjacent to $T$ in the graph defined by the Hasse diagram of the lattice are evaluated, that is, all elements that have Hamming distance one to $T$. Then, from those elements adjacent to $T$, we push into the stack every element $Y$ that satisfies the conditions: (i) $Y$ was not removed from the search space by the collections of restrictions; (ii) $Y$ has a cost less or equal than $T$; (iii) $Y$ is not in the stack. If no adjacent element of $T$ was added into the stack, then $T$ is considered a ``minimum exhausted'', which implies that it is popped from the stack and added to the collections of lower and upper restrictions, ending this iteration. The subroutine iterates until the stack is empty.

We will show now that $\proc{Minimum-Exhausting}$ has an error that may lead to
suboptimal results.  In Figure \ref{fig:ucurve_error}, we show an
example of an instance $\langle S, c \rangle$ of the U-curve problem,
in which the element of the Boolean lattice with minimum cost may be
lost during the execution of the $\proc{Minimum-Exhausting}$
subroutine. This subroutine is called in Figure
\ref{fig:ucurve_error:D}, in which the DFS procedure
starts pushing $11100$ into the stack. The search branches until the
element $01110$ is pushed into the stack; once $01110$ has no
adjacent element out of the stack, the subroutine backtracks to the
element $01111$ (Figure \ref{fig:ucurve_error:F}). The element $01111$
also has no element out of the stack, therefore it is popped from the
stack and included in the collections of restrictions -- this
operation removes from the search space the unique global minimum,
$00001$, which was not visited yet, thus losing it (Figure
\ref{fig:ucurve_error:G}).

\subsection{The principles of the proposed correction}
\label{sec:UCS!principles}

We will now present the principles of the proposed correction by providing sufficient conditions to remove unvisited elements from the search space without losing global minima. In the following, we show four sufficient conditions for removing elements from the search space.

\begin{prop}\label{sufficient_minimum_exhausted}
Let $\langle S, c \rangle$ be an instance of the U-curve problem, $\mathcal{X} \subseteq \mathcal{P}(S)$ be a current search space, and $X$ be an element of $\mathcal{X}$. If there exists an element $Y \in \mathcal{X}$ such that $Y \subseteq X$ and $c(Y) > c(X)$, then all elements in $[\emptyset, Y]$ have cost greater than $c(X)$. 
\end{prop}

\begin{proof}
Let us consider $X, Y \in \mathcal{P}(S)$ and $Z \in [\emptyset, Y]$
such that $Y \subseteq X$ and $c(Y) > c(X)$. By the definition of a
decomposable in U-shaped curves cost function, it holds that $c(Y) \le max \{ c(Z), c(X) \}$. Thus,  $c(Y) \le c(Z)$ or $c(Y) \le c(X)$, and $c(Y) > c(X)$. Therefore, we have $c(Z) \ge c(Y) > c(X)$, then that all elements in $[\emptyset, Y]$ have cost greater than $c(X)$.  
\end{proof}

The result of Proposition \ref{sufficient_minimum_exhausted} also holds for the Boolean lattice $(\mathcal{P}(S), \supseteq)$.

\begin{prop}\label{sufficient_minimum_exhausted_dual}
Let $\langle S, c \rangle$ be an instance of the U-curve problem, $\mathcal{X} \subseteq \mathcal{P}(S)$ be a current search space, and $X$ be an element of $\mathcal{X}$. If there exists an element $Y \in \mathcal{X}$ such that $X \supseteq Y$ and $c(X) > c(Y)$, then all elements in $[Y, S]$ have cost greater than $c(X)$. 
\end{prop}

\begin{proof}
Applying the principle of duality, the result of Proposition \ref{sufficient_minimum_exhausted} also holds for the Boolean lattice $(\mathcal{P}(S), \supseteq)$.
\end{proof}

\newpage

\begin{landscape}
\begin{figure}[!ht]
  \centering
  \begin{tabular}{ c c c}
    \subfigure[] {\scalebox{0.4}{\includegraphics[trim=3cm 8.5cm 1.5cm 2.5cm, clip=true]{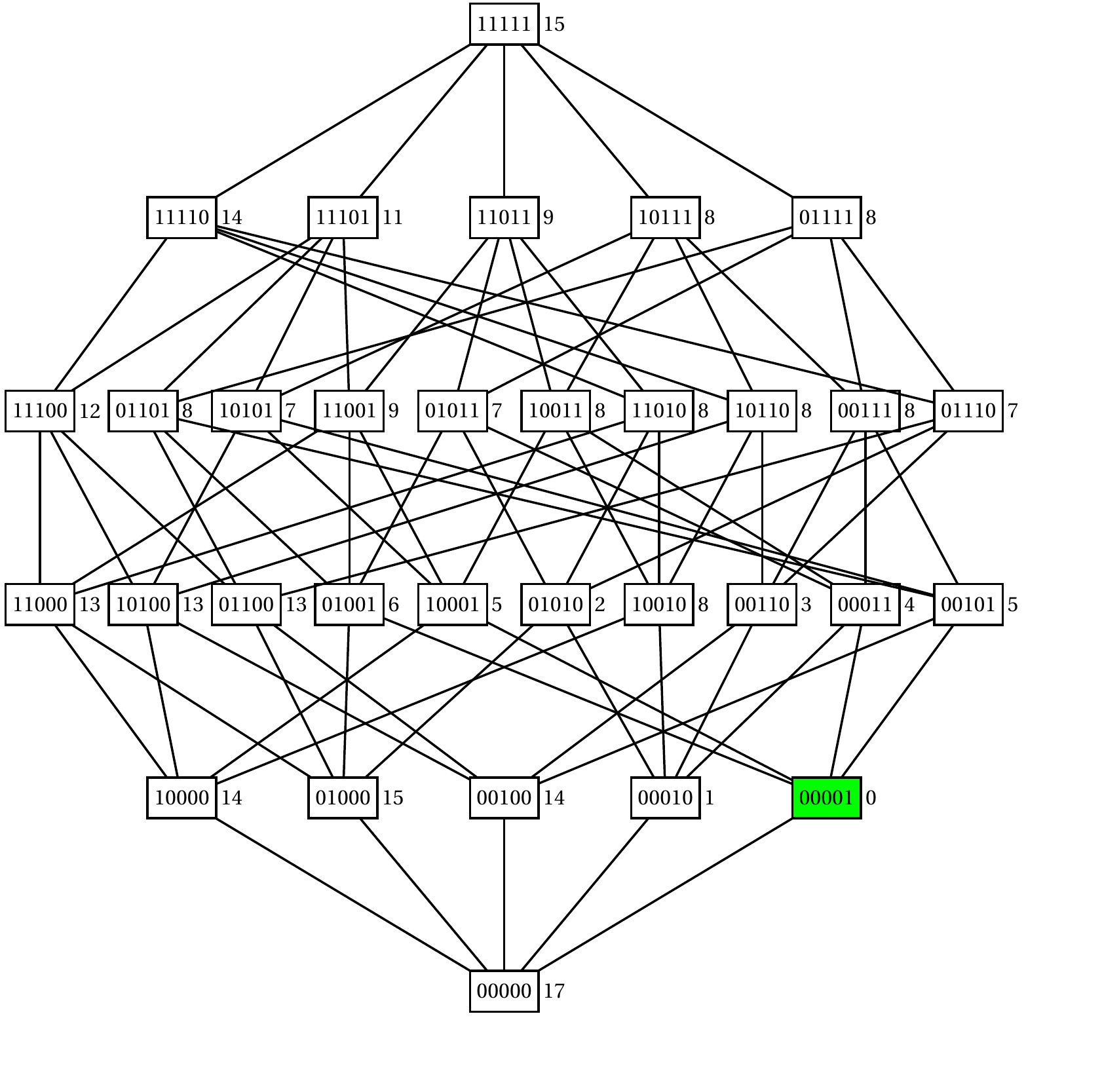}} \label{fig:ucurve_error:A} }
    &
    \subfigure[] {\scalebox{0.4}{\includegraphics[trim=3cm 8.5cm 1.5cm 2.5cm, clip=true]{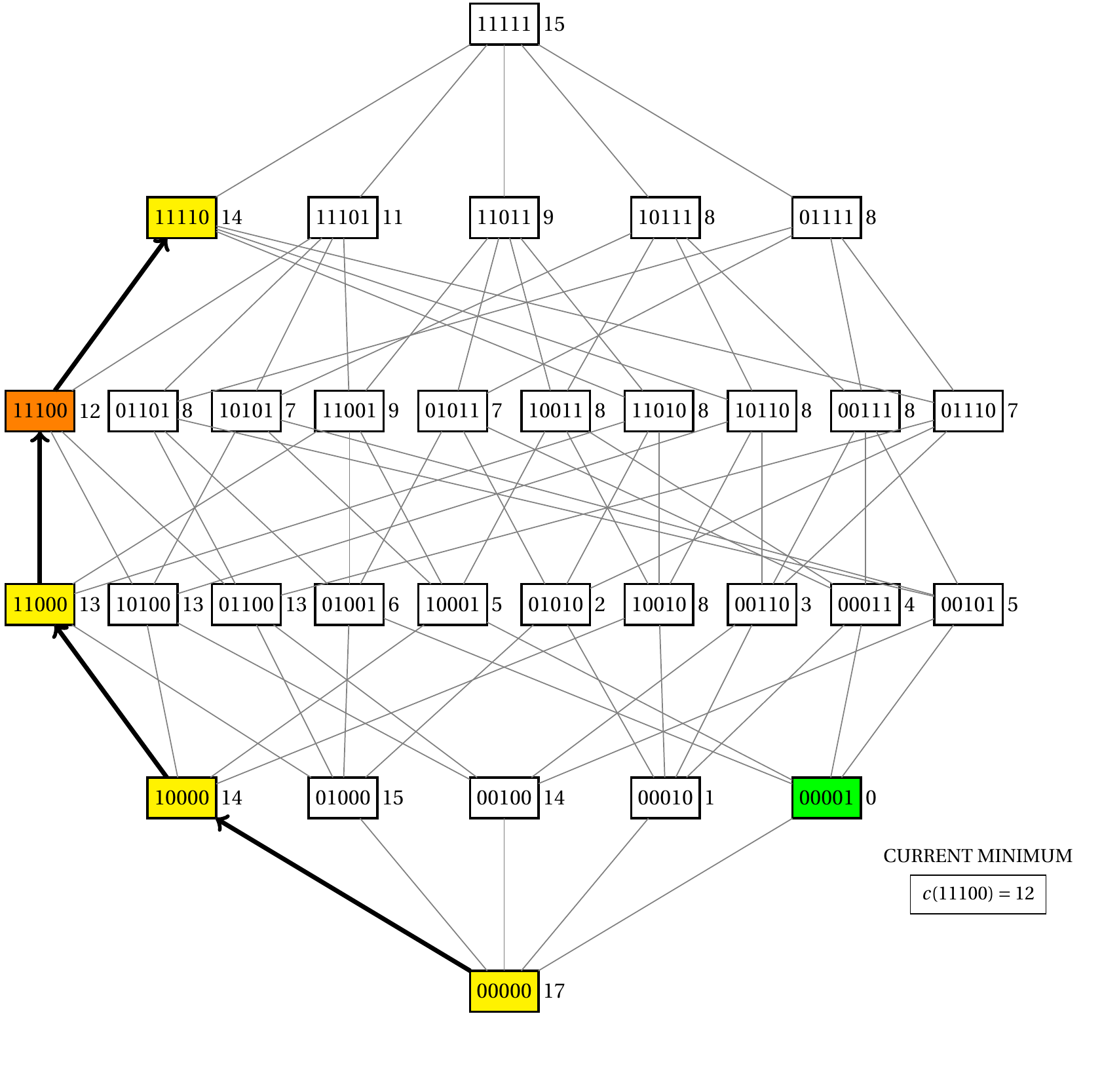}} \label{fig:ucurve_error:B} }
    &
    \subfigure[] {\scalebox{0.4}{\includegraphics[trim=3cm 8.5cm 1.5cm 2.5cm, clip=true]{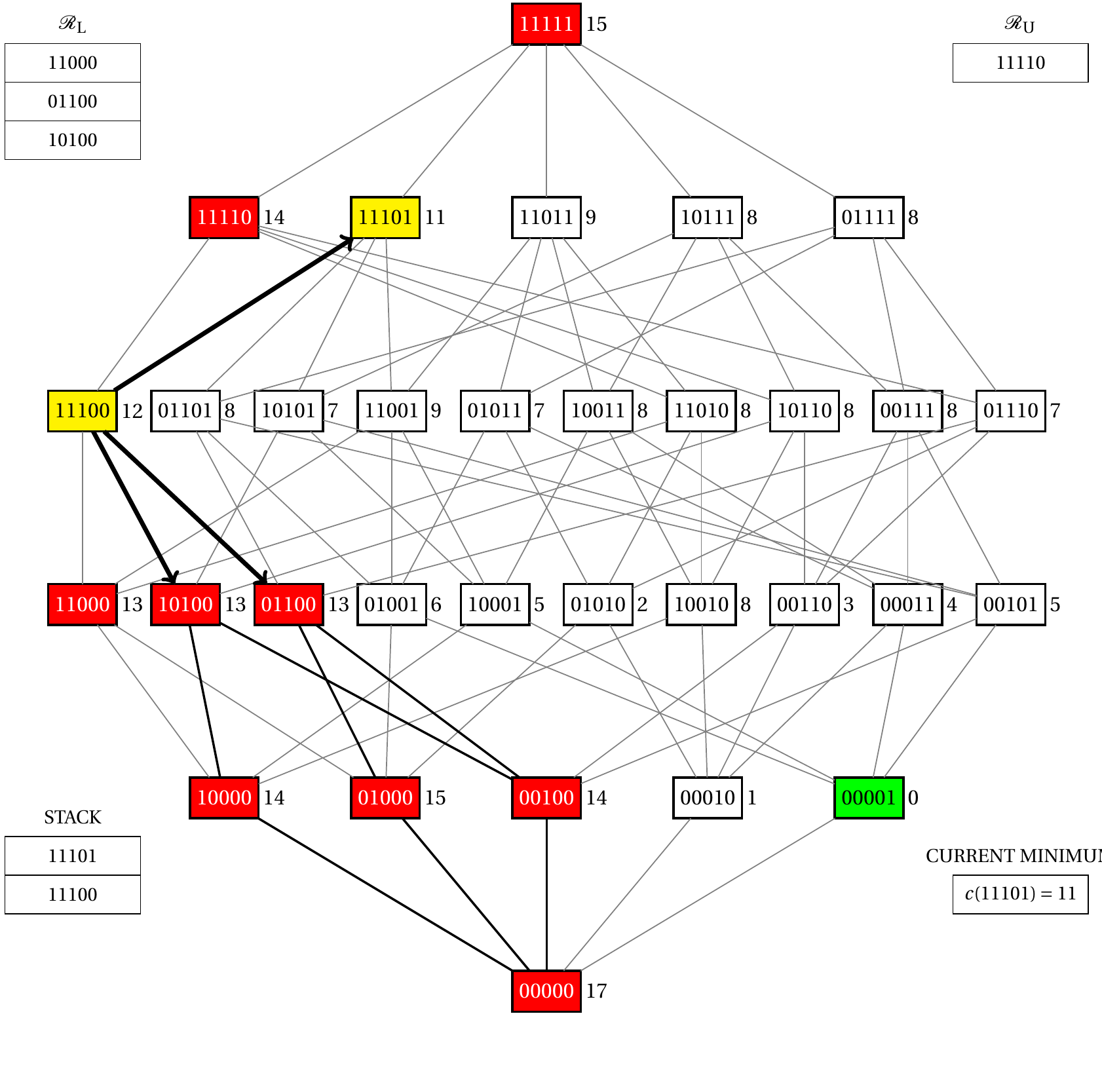}} \label{fig:ucurve_error:D} }
    \\
    \subfigure[] {\scalebox{0.4}{\includegraphics[trim=3cm 8.5cm 1.5cm 2.5cm, clip=true]{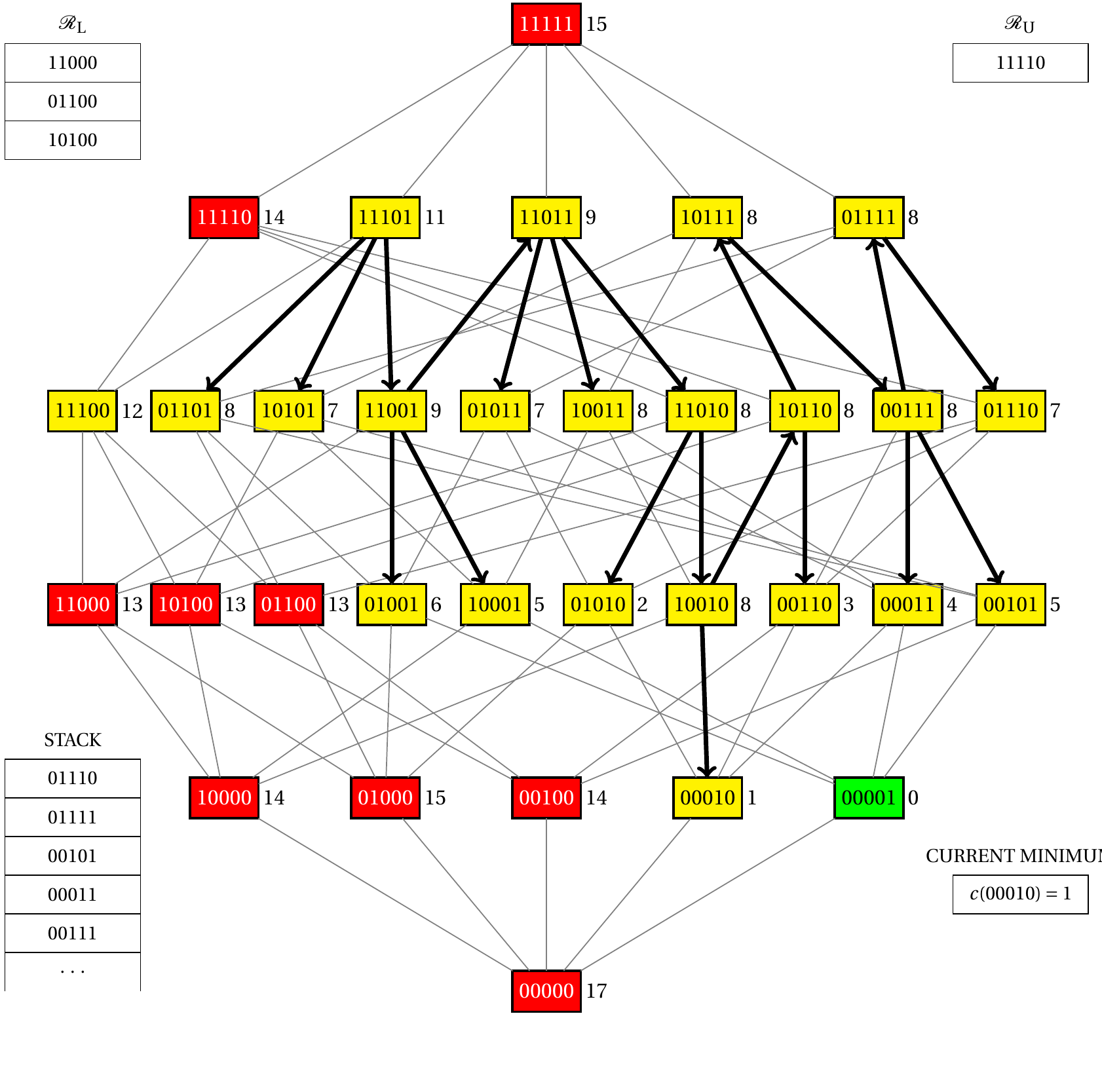}} \label{fig:ucurve_error:E} }
    &
    \subfigure[] {\scalebox{0.4}{\includegraphics[trim=3cm 8.5cm 1.5cm 2.5cm, clip=true]{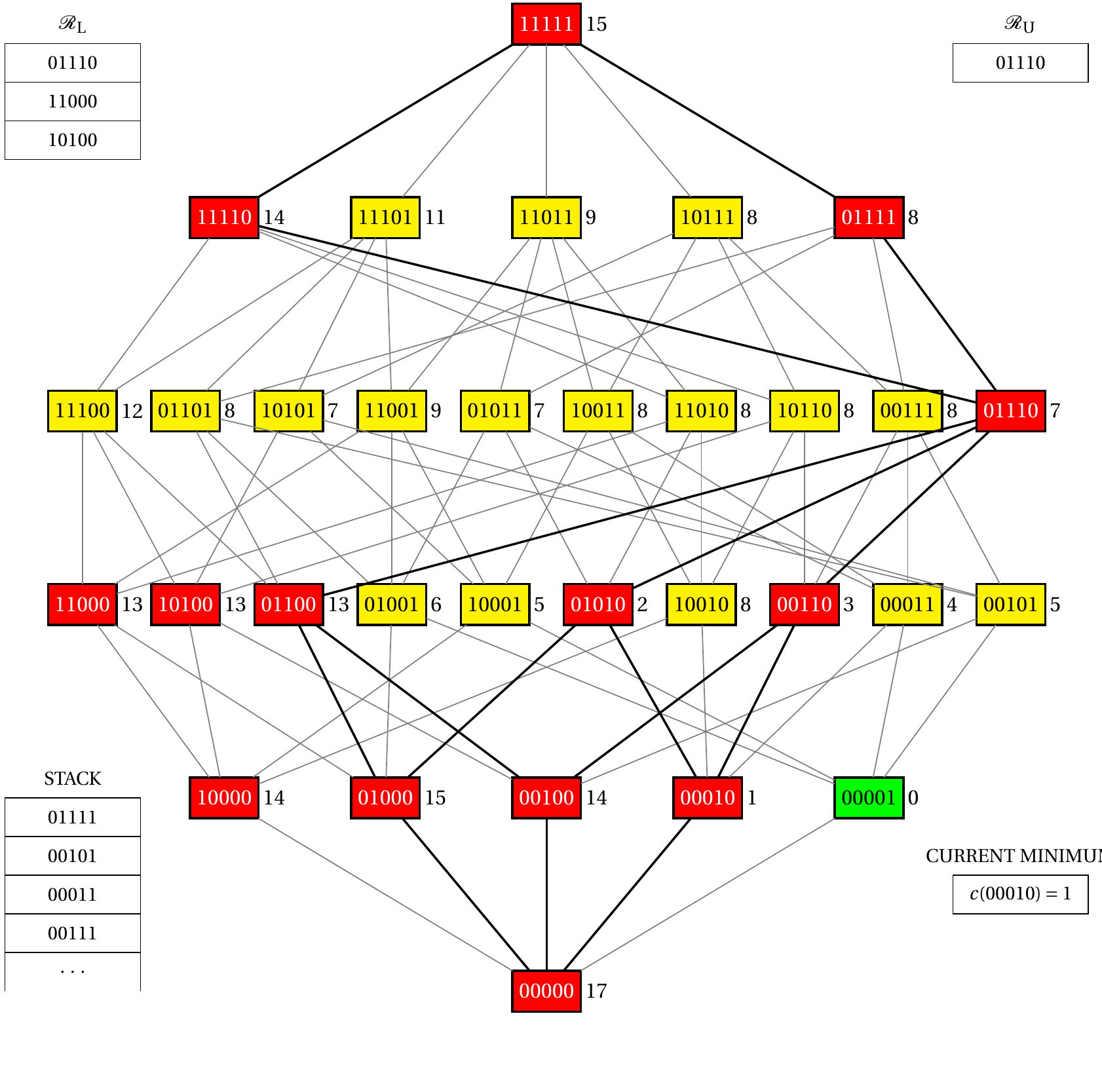}} \label{fig:ucurve_error:F} }
    &
    \subfigure[] {\scalebox{0.4}{\includegraphics[trim=2cm 8.5cm 1.5cm 2.5cm, clip=true]{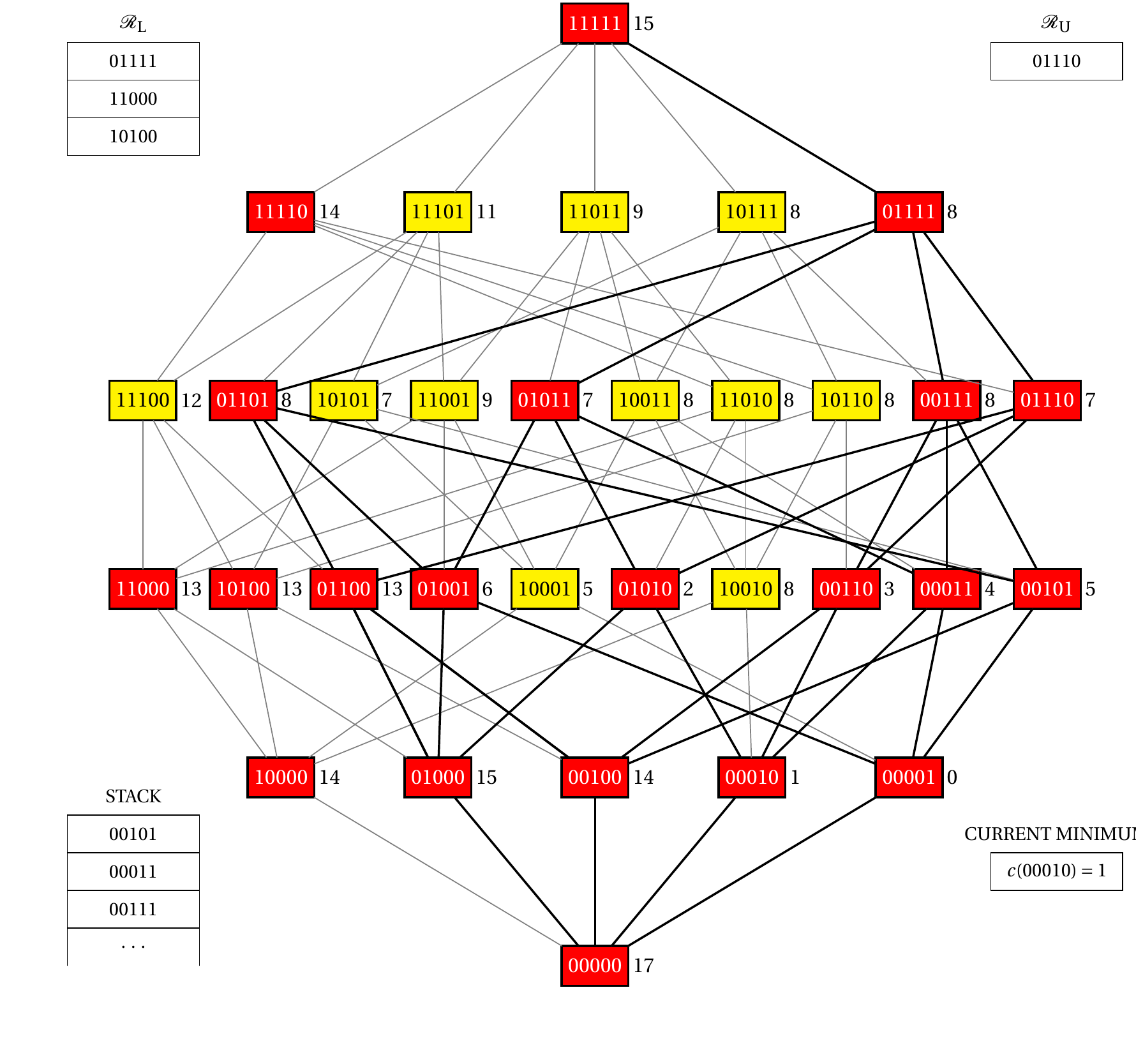}} \label{fig:ucurve_error:G} }
  \end{tabular}
  \caption{a counter-example of the correctness of
    $\proc{Minimum-Exhausted}$ subroutine. In this simulation of the subroutine, visited elements are 
    in yellow, while elements that were removed from the search space
    are in red. In Fig. \ref{fig:ucurve_error:G}, a pruning
  removes from the search space the global minimum, thus losing it.}
  \label{fig:ucurve_error} 
\end{figure}
\end{landscape}

Let $A$ and $B$ be elements of $\mathcal{P}(S)$. $A$ is {\bf lower
  adjacent} to $B$ (and $B$ is {\bf upper adjacent} to $A$) if $A
\subseteq B$ and there is no element $X \in \mathcal{P}(S)$ such that
$A \subset X \subset B$. There is a basic result in Boolean algebra
that also guarantees  the removal of one element from the search space
without the risk of losing global minima.

\begin{prop}\label{Boolean_algebra_minimum_exhausted}
Let $S$ be a non-empty set, $\mathcal{X} \subseteq \mathcal{P}(S)$ be a current search space, and $X$ be an element of $\mathcal{X}$. If for each element $Y \in \mathcal{P}(S)$ such that $Y$ is lower adjacent to $X$, $[\emptyset, Y] \cap \mathcal{X} = \emptyset$, then $[ \emptyset, X] - \{ X \}$ does not contain an element of $\mathcal{X}$.
\end{prop}

\begin{proof}
Let us consider an element $X \in \mathcal{X}$ and that, for each element $Y \in \mathcal{P}(S)$ such that $Y$ is lower adjacent to $X$, $[\emptyset, Y] \cap \mathcal{X} = \emptyset$. Due to the fact that $X$ is the least upper bound of its lower adjacent elements, $\cup \{ [\emptyset, Y] : Y \mbox{ is lower adjacent to }  X \} = [\emptyset, X] - X$. Therefore, $[ \emptyset, X] - \{ X \}$ does not contain an element of $\mathcal{X}$.
\end{proof}

The result of Proposition \ref{Boolean_algebra_minimum_exhausted} also holds for the Boolean lattice $(\mathcal{P}(S), \supseteq)$.

\begin{prop}\label{Boolean_algebra_minimum_exhausted_dual}
Let $S$ be a non-empty set, $\mathcal{X} \subseteq \mathcal{P}(S)$ be a current search space, and $X$ be an element of $\mathcal{X}$. If for each element $Y \in \mathcal{P}(S)$ such that $Y$ is upper adjacent to $X$, $[Y, S] \cap \mathcal{X} = \emptyset$, then $[X, S] - \{ X \}$ does not contain an element of $\mathcal{X}$.
\end{prop}

\begin{proof}
Applying the principle of duality, the result of Proposition \ref{Boolean_algebra_minimum_exhausted} also holds for the Boolean lattice $(\mathcal{P}(S), \supseteq)$.
\end{proof}

To fix the error in the $\proc{Minimum-Exhausting}$ subroutine, we
developed a new optimal search algorithm for the U-curve problem.
This new algorithm stores an element that is visited during a
search in a structure called node. A node contains an element of $\mathcal{P}(S)$ and two Boolean
flags: ``lower restriction flag'' and ``upper restriction flag'',
which assign whether the element of the node was removed from the
search space, respectively, by the collection of lower restrictions or
by the collection of upper restrictions. The algorithm also takes into
account Propositions \ref{sufficient_minimum_exhausted},
\ref{sufficient_minimum_exhausted_dual},
\ref{Boolean_algebra_minimum_exhausted} and
\ref{Boolean_algebra_minimum_exhausted_dual} to remove elements from
the search space. There are two types of removal of elements from the search space:
\begin{itemize}
\item let $X$ and $Y$ be elements  in  $\mathcal{X} \subseteq \mathcal{P}(S)$, such that $Y$ is lower  (in the dual case, upper) adjacent to $X$. If $c(Y) > c(X)$, then by Proposition \ref{sufficient_minimum_exhausted} (\ref{sufficient_minimum_exhausted_dual}) the interval $[\emptyset, Y]$ ($[Y, S]$) should be removed from the search space. Therefore, $Y$ is included in the collection of lower (upper) restrictions and its node is marked with the lower (upper) restriction flag. 
\item let $X$ be an element  in  $\mathcal{X} \subseteq \mathcal{P}(S)$. If there is no element lower (in the dual case, upper) adjacent to $X$ in $\mathcal{X}$, then by Proposition \ref{Boolean_algebra_minimum_exhausted} (\ref{Boolean_algebra_minimum_exhausted_dual}) the interval $[\emptyset, X]$ ($[X, S]$) should be removed from the search space. Therefore, $X$ is handled as the element $Y$ of the previous item.
\end{itemize}

In Figure~\ref{fig:ucurve_corrected},  we show the first steps of a simulation of this new algorithm; the
complete simulation is provided in the Supplementary Material.

% UCS
%
\subsection{The $\proc{U-Curve-Search}$ algorithm}
\label{sec:UCS:algorithm}

In this section, we will present the $\proc{U-Curve-Search}$
($\proc{UCS}$) algorithm, which implements the proposed correction
presented in the previous section. We will also give the time
complexity of the main algorithm. The description of this algorithm
will be done following a logical decomposition of the algorithm in a
bottom-up way, that is, the subroutines will be presented and studied
in the inverse order of the subroutine order call tree. The
presentation of each subroutine will be given through its description and
pseudo-code; the descriptions of dual subroutines will be omitted, as
well as the pseudo-codes of very simple subroutines.

From now on, the cardinality of a non-empty set $S$ will be denoted as $n := |S|$. It is assumed that any element $X \in \mathcal{P}(S)$ demands $n$ bits to represent it, and that any collection $\mathcal{R} \subseteq \mathcal{P}(S)$ is implemented using a doubly linked list. Therefore, the time complexity of a search is $O(n |\mathcal{R}| )$, of an insertion is $O(1)$, and of a deletion is $O(1)$. Finally, we assume that any call of the cost function $c$ demands $O(f(n))$, in which $f(n)$ depends on the chosen cost function.

\subsubsection{Update of upper and lower restrictions}

Let $X$ be an element in  $\mathcal{P}(S)$. $X$ is {\bf covered} by a collection of lower (in the dual form, upper) restrictions $\mathcal{R}_L$ ($\mathcal{R}_U$)  if there exists an element $R \in \mathcal{R}_L$ ($R \in \mathcal{R}_U$) such that $X \subseteq R$ ($R \subseteq X$). 

\paragraph{Subroutine description} 
$\proc{Update-Lower-Restriction}$ receives an element $X \in \mathcal{P}(S)$ and a collection of lower restrictions $\mathcal{R}_L$. If $X$ is not covered by $\mathcal{R}_L$, then the algorithm adds $X$ into $\mathcal{R}_L$ and removes all the elements of $\mathcal{R}_L$ that are properly contained in $X$. Finally, the subroutine returns the updated collection $\mathcal{R}_L$.

\newpage

\begin{landscape}
\begin{figure}[!ht]
  \centering
  \begin{tabular}{c c c}
    \subfigure[] {\scalebox{0.4}{\includegraphics[trim=3cm 8.5cm 1.5cm 2.5cm, clip=true]{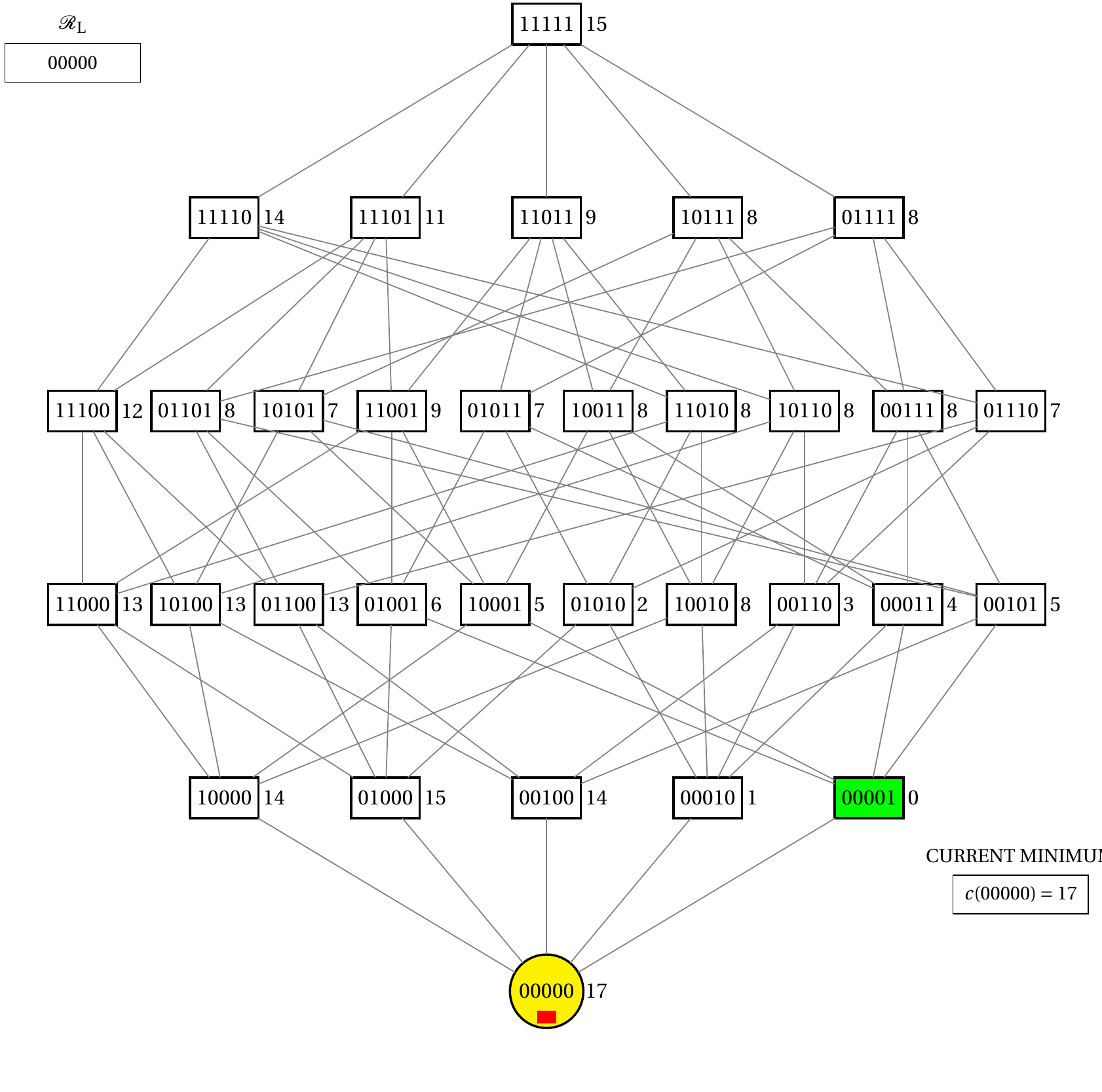}} \label{fig:ucurve_corrected:B} }
    &
    \subfigure[] {\scalebox{0.4}{\includegraphics[trim=3cm 8.5cm 1.5cm 2.5cm, clip=true]{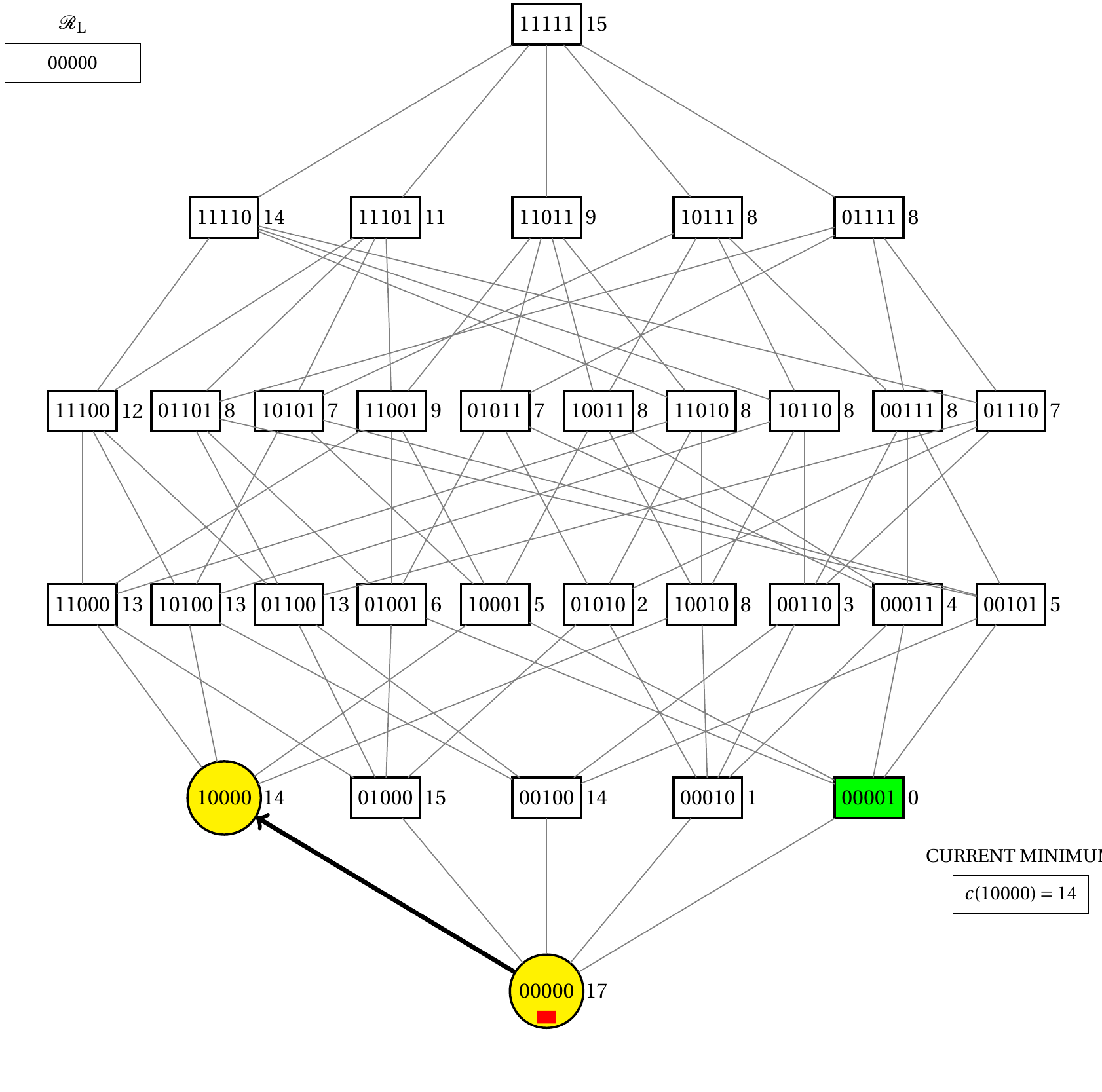}} \label{fig:ucurve_corrected:C} }
    &
    \subfigure[] {\scalebox{0.4}{\includegraphics[trim=3cm 8.5cm 1.5cm 2.5cm, clip=true]{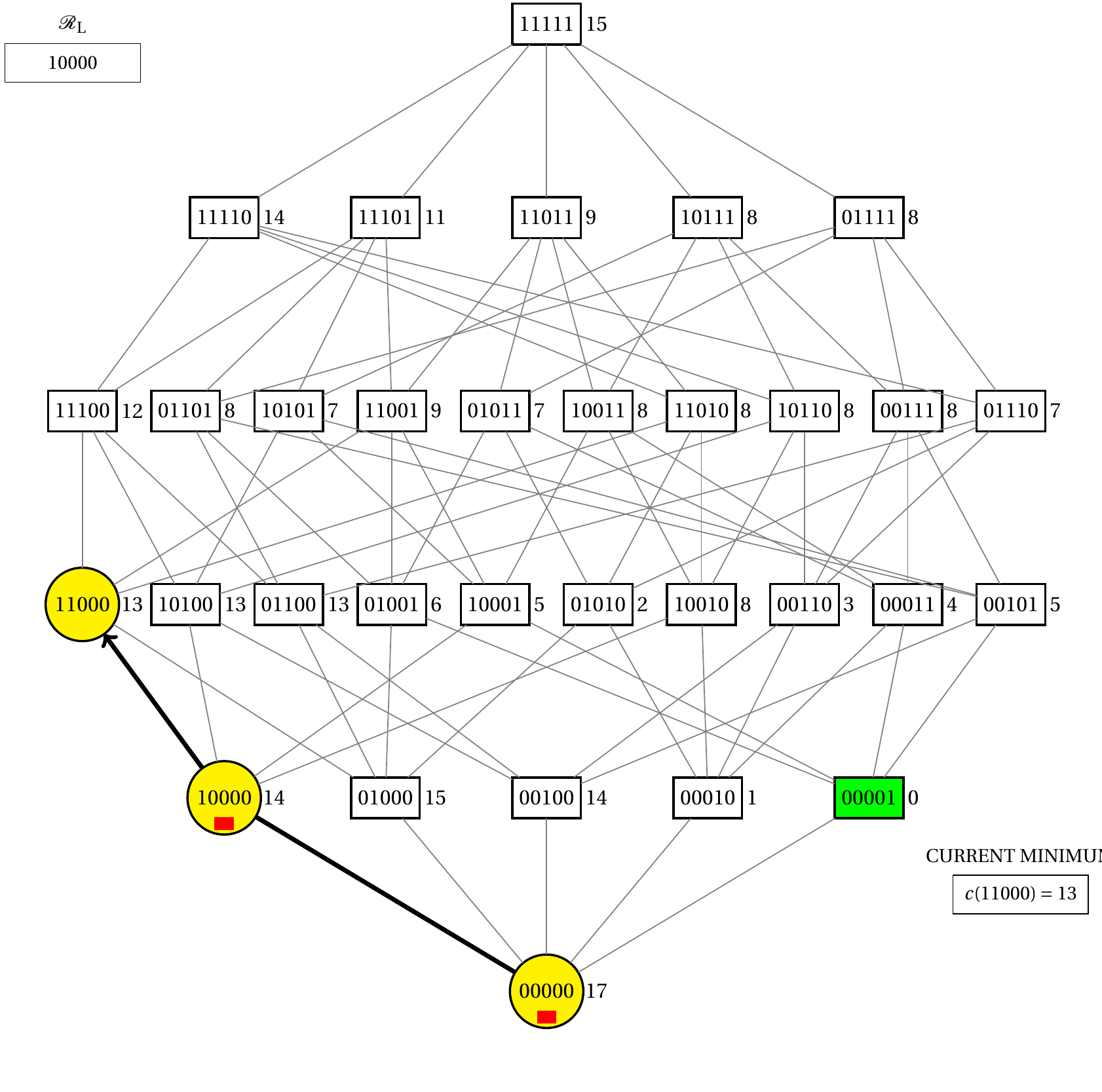}} \label{fig:ucurve_corrected:D} }
    \\
    \subfigure[] {\scalebox{0.4}{\includegraphics[trim=3cm 8.5cm 1.5cm 2.5cm, clip=true]{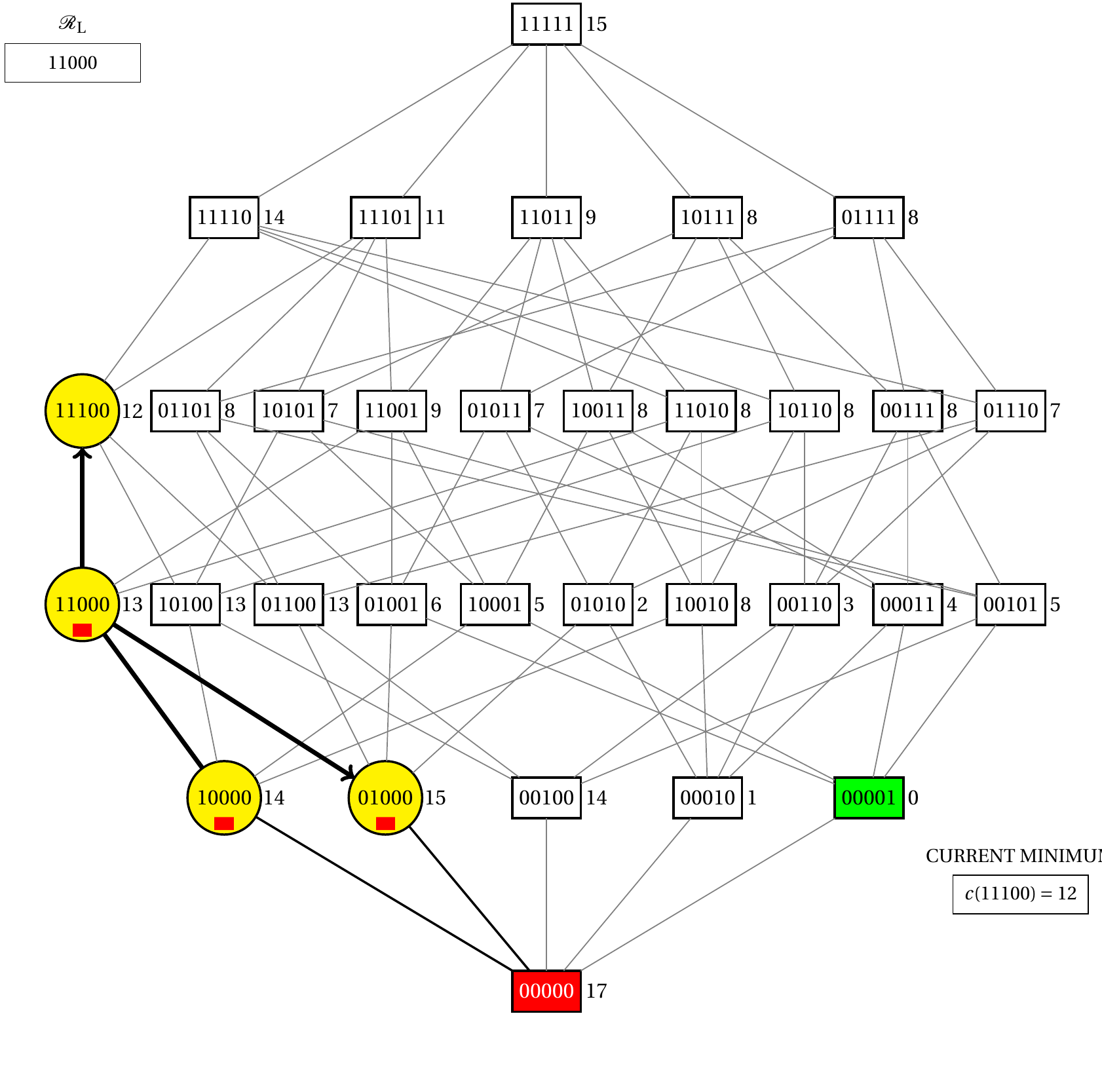}} \label{fig:ucurve_corrected:F} }
    &
    \subfigure[] {\scalebox{0.4}{\includegraphics[trim=3cm 8.5cm 1.5cm 2.5cm, clip=true]{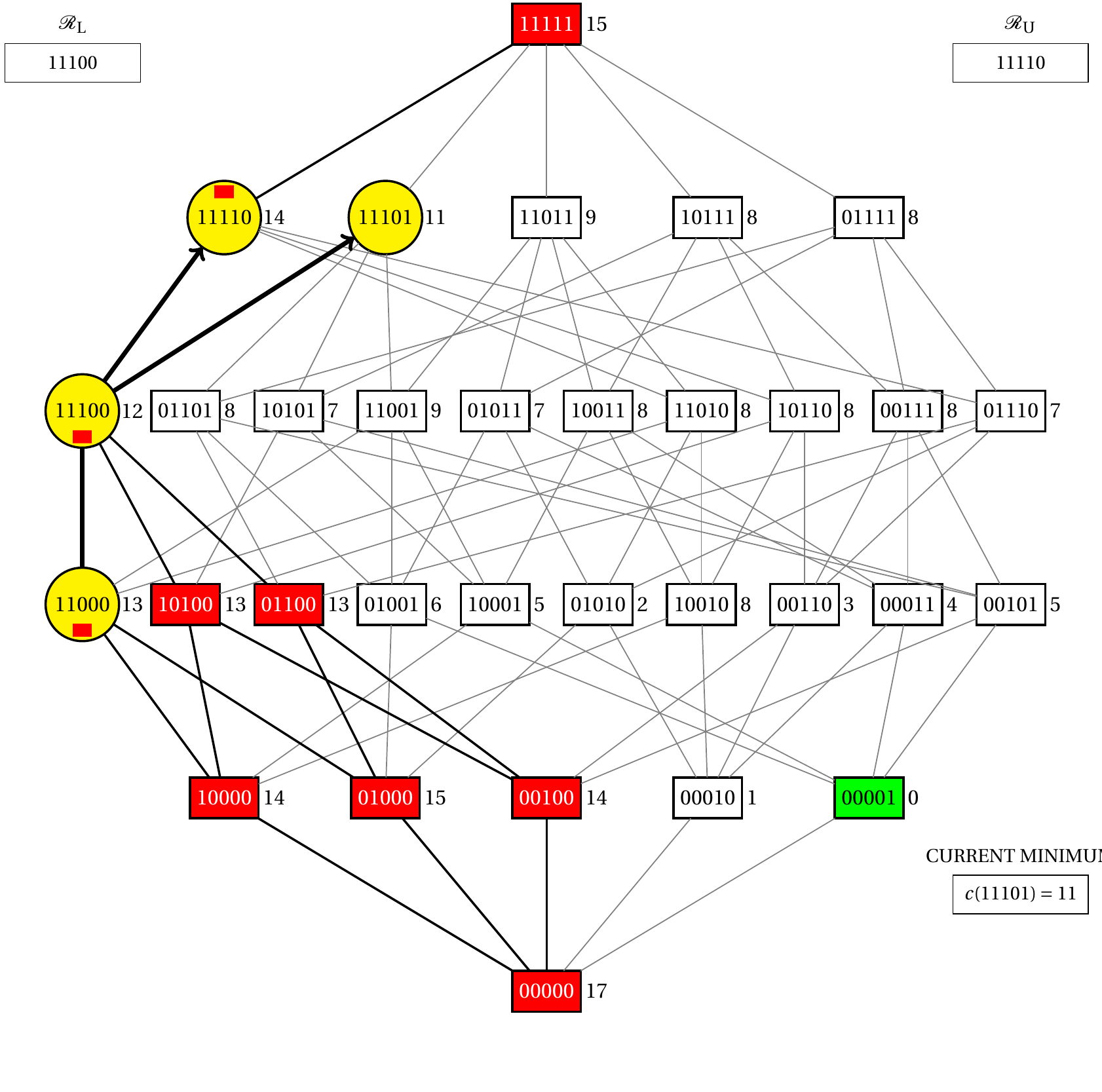}} \label{fig:ucurve_corrected:H} }
    &
    \subfigure[] {\scalebox{0.4}{\includegraphics[trim=3cm 8.5cm 1.5cm 2.5cm, clip=true]{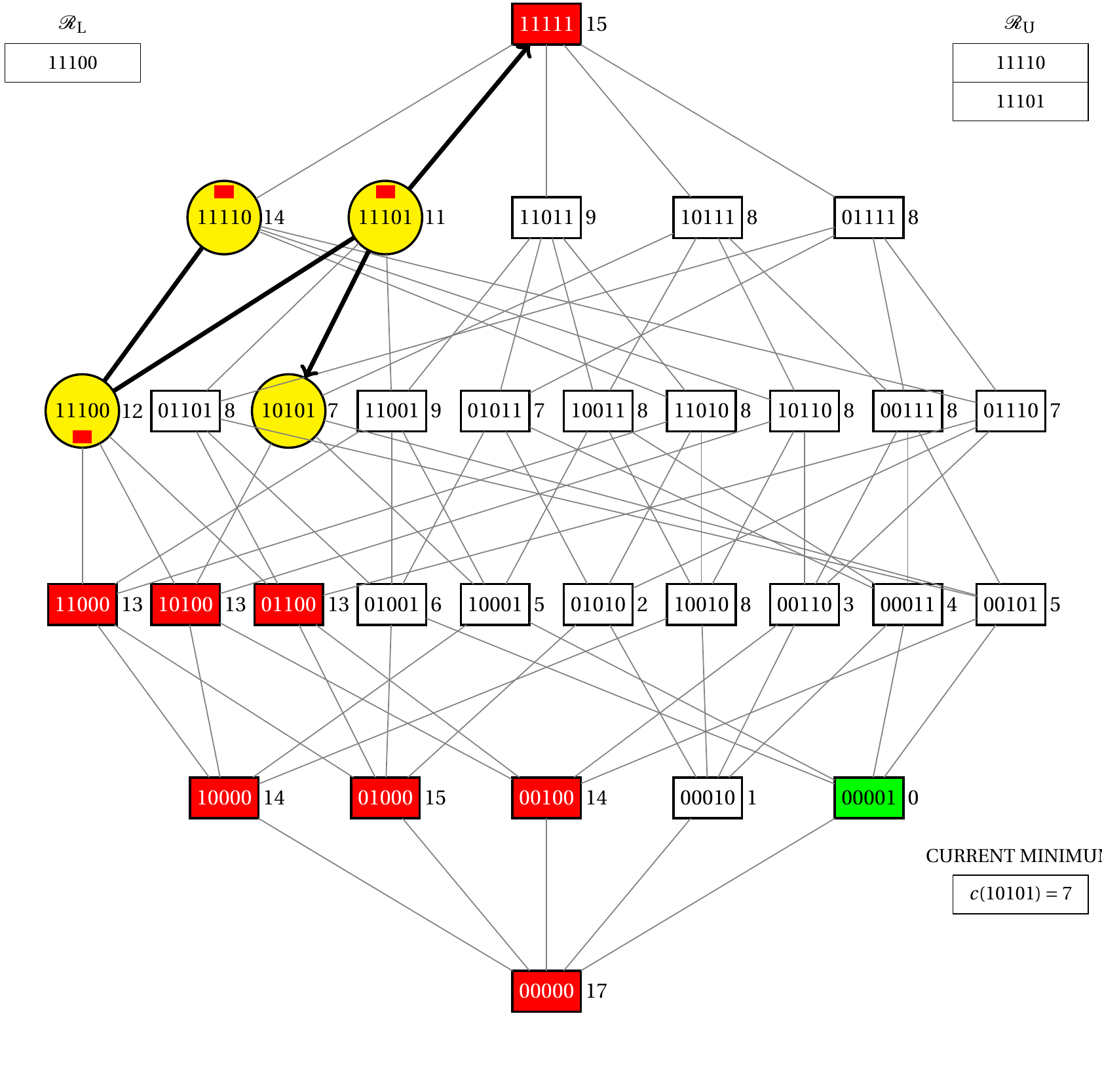}} \label{fig:ucurve_corrected:J} }
  \end{tabular}
  \caption{a simulation of the $\proc{UCS}$
    algorithm. \ref{fig:ucurve_corrected:B}: the selection of a
    minimal element in the search space;  \ref{fig:ucurve_corrected:C}--
    \ref{fig:ucurve_corrected:J}: first five iterations of the DFS
    subroutine; at each iteration, a visited element is either
    included in the collection of nodes (yellow circles) or is used
    for pruning. A lower (upper) red bar inside a yellow circle
    assigns that the node's element is in the collection of
    lower (upper) restrictions $\mathcal{R}_L$ ($\mathcal{R}_U$).}
  \label{fig:ucurve_corrected} 
\end{figure}
\end{landscape}

\subsubsection{Minimal and maximal elements}

Let $X$ be an element of $\mathcal{A} \subseteq \mathcal{P}(S)$. $X$ is {\bf minimal} (in the dual form, {\bf maximal}) in $\mathcal{A}$ if for any $Y \in \mathcal{A}$, $Y \subseteq X$ ($X \subseteq Y$) implies that $X = Y$.

Now let $\mathcal{X}$ be a function that takes values from $\mathcal{P}(S) \times \mathcal{P}(S)$ to $\mathcal{P}(S)$, defined as
\begin{center}
$\mathcal{X}(\mathcal{R}_L, \mathcal{R}_U) =\mathcal{P}(S) - \bigcup \{ [\emptyset, R] : R \in \mathcal{R}_L\} - \bigcup \{ [R, S] : R \in \mathcal{R}_U\}$.
\end{center}
We say that $\mathcal{X}(\mathcal{R}_L, \mathcal{R}_U)$ is a {\bf current search space} of the search space $\mathcal{P}(S)$.

\paragraph{Subroutine description} 
$\proc{Minimal-Element}$ receives a set $S$ and a collection of lower restrictions $\mathcal{R}_L$. If the current search space $\mathcal{X}(\mathcal{R}_L, \emptyset)$ is not empty, then the subroutine returns an element $X$ that is minimal in $\mathcal{X}(\mathcal{R}_L, \emptyset)$; otherwise, it returns $NIL$.

\citet{Ris:2010} provided an implementation of this subroutine (Algorithm
$4$).

\subsubsection{Depth-first search}
Now, we will describe the depth-first search procedure that is called
at each iteration of the $\proc{UCS}$ algorithm. We will start
describing a data structure to store elements of
$\mathcal{P}(S)$ and the flags described in Section
\ref{sec:UCS!principles}. A {\bf node} is a structure with four
variables, of which all of them map to an element  of $\mathcal{P}(S)$:
\begin{itemize}
\item ``element'' represents an element of $\mathcal{P}(S)$;

\item ``unverified'' is used by the search algorithm to keep track of which element adjacent to ``element'' in the Boolean lattice was verified or not;

\item ``lower\_adjacent'' and ``upper\_adjacent'' represent the
  topology of the current search space. If they are empty, then they are equivalent to, respectively, the ``lower restriction flag'' and ``upper restriction flag'' presented in Section \ref{sec:UCS!principles};
\end{itemize}

Let $\mathcal{G}$ be a collection of nodes, $\mathbf{Y}$ be a node in
$\mathcal{G}$, and $\mathcal{X}(\mathcal{R}_L, \mathcal{R}_U)$ be a
current search space. Now, we will define three rules of the kernel of this subroutine dynamics:

\begin{enumerate}[label=(\roman{*}), ref=(\roman{*})]
\item{\it Criterion to stop a search.} \label{item:DFS!unvisited} If there is an element $X \in \mathcal{X}(\mathcal{R}_L, \mathcal{R}_U)$ adjacent to $\mathbf{Y}[element]$ such that $\mathbf{X}[element] \ne X$ for all node $\mathbf{X}$ in $\mathcal{G}$, then $X$ {\bf is an unvisited adjacent element to} $\mathbf{Y}[element]$;

\item{\it Management of the restriction
    flags.} \label{item:DFS!flag_1} Given a $X$ in $\mathcal{P}(S)$ such that $X$ is lower (upper) adjacent to $\mathbf{Y}[element]$, if the unique element in $\mathbf{Y}[element] - X$ ($X - \mathbf{Y}[element]$) is not in $\mathbf{Y}[lower\_adjacent]$ ($\mathbf{Y}[upper\_adjacent]$), then $\mathcal{R}_L$ ($\mathcal{R}_U$) covers $X$.

This rule describes an implementation of the ``upper restriction flag'' and ``lower restriction flag'' presented in Section \ref{sec:UCS!principles}, in a way that they also store information for the verification of the conditions of Propositions \ref{Boolean_algebra_minimum_exhausted} and \ref{Boolean_algebra_minimum_exhausted_dual};

\item{\it Management of visited adjacent elements.} \label{item:DFS!control} If $X$ is an unvisited adjacent element to $\mathbf{Y}$, then $\mathbf{Y}[unverified]$ has an element $y$ such that either $X = \mathbf{Y}[element] \cup \{ y \}$ or $X = \mathbf{Y}[element] - \{ y \}$ holds.

This rule defines a way to avoid that, for each time it is verified if $\mathbf{Y}$ has an unvisited adjacent element, one must inspect all adjacent elements to $\mathbf{Y}[element]$ in $\mathcal{P}(S)$;
\end{enumerate}

We will present now three subroutines that are used by the main DFS
algorithm: the first one provides selection of an unvisited adjacent
element; the second one provides lower pruning; the third one provides
a pruning that takes into account the
conditions of Propositions \ref{sufficient_minimum_exhausted} and
\ref{sufficient_minimum_exhausted_dual}.

\paragraph{Subroutine description} 
$\proc{Select-Unvisited-Adjacent}$ receives a node $\mathbf{Y}$, a
collection of nodes $\mathcal{G}$, a set $S$, a collection of lower
restrictions $\mathcal{R}_L$, and a collection of upper restrictions
$\mathcal{R}_U$. This subroutine searches for an unvisited adjacent
element to $\mathbf{Y}[element]$. It verifies the condition of rule \ref{item:DFS!unvisited}: if an unvisited adjacent $X$ is found, then this subroutine returns a node containing it; otherwise, it returns $NIL$. During the search procedure, the flags of $\mathbf{Y}$ are updated following the rules \ref{item:DFS!flag_1} and \ref{item:DFS!control}; hence, this subroutine also returns $\mathbf{Y}$.

\newpage

\begin{center}
\begin{codebox} 
\Procname{$\proc{Select-Unvisited-Adjacent}(\mathbf{Y}, \mathcal{G}, S, \mathcal{R}_L, \mathcal{R}_U)$} 
\li      \While $\mathbf{Y}[unverified] \ne \emptyset$  \label{li:unvisited:while_begin}
\li          \Do   remove an element $y$ from $\mathbf{Y}[unverified]$    \label{li:select:unverified}
\li                   \If  $y$ is in $\mathbf{Y}[element]$
\li                                \Then  $X \gets \mathbf{Y}[element] - \{ y \}$
\li                                \Else   $X \gets \mathbf{Y}[element] \cup \{ y \}$ \label{li:select:unverified_end}
                                    \End
\li                  \If $X$ is an unvisited adjacent to $\mathbf{Y}[element]$  $\phantom{abcdef}$ \Comment $\mathcal{G}$ is used in this test \label{li:unvisited:node_verification}
\li                                       \Then $\mathbf{X}[element] \gets \mathbf{X}[lower\_adjacent] \gets X$  \label{li:unvisited:node}
\li                                                 $\mathbf{X}[upper\_adjacent] \gets S - X$
\li                                                 $\mathbf{X}[unverified] \gets S$ \label{li:unvisited:unvisited_2}
\li                                                 \Return $\langle \mathbf{X}, \mathbf{Y} \rangle$  \label{li:unvisited:return}
\li                      \Else \If $X$ is lower adjacent to $\mathbf{Y}[element]$ {\bf and} $X$ is covered by $\mathcal{R}_L$ \label{li:unvisited:RL_verification}
\li                            \Then $\mathbf{Y}[lower\_adjacent] \gets \mathbf{Y}[lower\_adjacent] -  (\mathbf{Y}[element] - X)$
                               \End
\li                      \If $X$ is upper adjacent to $\mathbf{Y}[element]$ {\bf and} $X$ is covered by $\mathcal{R}_U$  \label{li:unvisited:RU_verification}
\li                           \Then $\mathbf{Y}[upper\_adjacent] \gets \mathbf{Y}[upper\_adjacent] -  (X - \mathbf{Y}[element])$
                               \End  
                          \End % if     
              \End                             \label{li:unvisited:while_end}   % while
\li     \Return $\langle NIL, \mathbf{Y} \rangle$ 
\end{codebox}\end{center}
\onehalfspacing

\paragraph{Subroutine description} 
$\proc{Lower-Pruning}$ receives a node  $\mathbf{Y}$, a collection of nodes $\mathcal{G}$, and a collection of lower restrictions $\mathcal{R}_L$. This subroutine includes $\mathbf{Y}[element]$ in $\mathcal{R}_L$ using the $\proc{Update-Lower-Restriction}$ subroutine. It also removes from $\mathcal{G}$ every node $\mathbf{X}$ such that $\mathbf{X}[element]$ is properly contained in $\mathbf{Y}[element]$. Finally, this subroutine returns the updated collections $\mathcal{G}$ and $\mathcal{R}_L$.

\paragraph{Subroutine description} 
$\proc{Node-Pruning}$ receives a pair of nodes $\mathbf{X}$ and
$\mathbf{Y}$, a collection of nodes $\mathcal{G}$, a collection of
lower restriction $\mathcal{R}_L$, a collection of upper restrictions
$\mathcal{R}_U$, and a decomposable in U-shaped
curves cost function $c$. It verifies the conditions of Propositions
\ref{sufficient_minimum_exhausted} and
\ref{sufficient_minimum_exhausted_dual}, pruning the collection of nodes and updating the collections of restrictions accordingly. Finally, it returns the updated collections $\mathcal{G}$, $\mathcal{R}_L$, and $\mathcal{R}_U$.

\begin{codebox}
\Procname{$\proc{Node-Pruning}(\mathbf{X}, \mathbf{Y}, \mathcal{G}, \mathcal{R}_L, \mathcal{R}_U, c)$} 
\li            $X \gets \mathbf{X}[element] \phantom{abcde} Y \gets \mathbf{Y}[element]$
\li           \If $X$ is upper adjacent to $Y$ {\bf and } $c(X) < c(Y)$   \label{li:graph_pruning_1}
\li                \Then  $\langle \mathcal{G}, \mathcal{R}_L \rangle \gets \proc{Lower-Pruning}(\mathbf{Y}, \mathcal{G}, \mathcal{R}_L)$ \label{li:graph_pruning_1_1}        
\li                           $\mathbf{X}[lower\_adjacent] \gets \mathbf{X}[lower\_adjacent] -  (X - Y)$   \label{li:graph_pruning_lower_single_1}
\li                           $\mathbf{Y}[lower\_adjacent] \gets \emptyset$   \label{li:graph_pruning_lower_empty_1}
\li           \ElseIf $X$ is lower adjacent to $Y$  {\bf and} $c(X) < c(Y)$   \label{li:graph_pruning_2}
\li                 \Then $\langle \mathcal{G}, \mathcal{R}_U \rangle \gets \proc{Upper-Pruning}(\mathbf{Y}, \mathcal{G}, \mathcal{R}_U)$  \label{li:graph_pruning_2_1}
\li                           $\mathbf{X}[upper\_adjacent] \gets \mathbf{X}[upper\_adjacent] -  (Y - X)$ \label{li:graph_pruning_upper_single_1}
\li                           $\mathbf{Y}[upper\_adjacent] \gets \emptyset$    \label{li:graph_pruning_upper_empty_1}
\li           \ElseIf $Y$ is upper adjacent to $X$ {\bf and} $c(X) > c(Y)$  \label{li:graph_pruning_3}
\li                 \Then $\langle \mathcal{G}, \mathcal{R}_L \rangle \gets \proc{Lower-Pruning}(\mathbf{X}, \mathcal{G}, \mathcal{R}_L)$  \label{li:graph_pruning_3_1}
\li                           $\mathbf{Y}[lower\_adjacent] \gets \mathbf{Y}[lower\_adjacent] -  (Y - X)$ \label{li:graph_pruning_lower_single_2}
\li                           $\mathbf{X}[lower\_adjacent] \gets \emptyset$  \label{li:graph_pruning_lower_empty_2}
\li          \ElseIf $Y$ is lower adjacent to $X$ {\bf and} $c(X) > c(Y)$  \label{li:graph_pruning_4}
\li                 \Then $\langle \mathcal{G}, \mathcal{R}_U \rangle \gets \proc{Upper-Pruning}(\mathbf{X}, \mathcal{G}, \mathcal{R}_U)$  \label{li:graph_pruning_4_1}
\li                           $\mathbf{Y}[upper\_adjacent] \gets \mathbf{Y}[upper\_adjacent] -  (X - Y)$  \label{li:graph_pruning_upper_single_2}
\li                           $\mathbf{X}[upper\_adjacent] \gets \emptyset$  \label{li:graph_pruning_upper_empty_2}
                     \End
\li      \Return $\langle \mathcal{G}, \mathcal{R}_L, \mathcal{R}_U \rangle$
\end{codebox}
\onehalfspacing

\paragraph{The DFS main subroutine} Let $X$ be an element of
$\mathcal{P}(S)$ and $c$ be a cost function decomposable in U-shaped
curves. We want to perform a depth-first search in
$\mathcal{X}(\mathcal{R}_L, \mathcal{R}_U)$ that starts with an
element $Y$ and has as the depth-first criterion to expand the first element
$X$  adjacent to $Y$ such that $X \in \mathcal{X}(\mathcal{R}_L, \mathcal{R}_U)$ and $c(X) \le c(Y)$.

\paragraph{Subroutine description} 
$\proc{DFS}$ receives a non-empty set $S$, a node $\mathbf{M}$, a
collection of lower restrictions $\mathcal{R}_L$, a collection of
upper restrictions $\mathcal{R}_U$, and a decomposable in U-shaped curves cost function $c$. This subroutine performs a
depth-first search using the criteria described above. The algorithm inspects a subset of the current search space $\mathcal{X}(\mathcal{R}_L, \mathcal{R}_U)$, update the current minima, and update the current search space to $\mathcal{X}(\mathcal{R'}_L, \mathcal{R'}_U) \subseteq \mathcal{X}(\mathcal{R}_L, \mathcal{R}_U)$. Therefore, the algorithm returns $\mathcal{R'}_L$, $\mathcal{R'}_U$, and a collection $\mathcal{M}$ containing every element $X$ in $\mathcal{X}(\mathcal{R}_L, \mathcal{R}_U) - \mathcal{X}(\mathcal{R'}_L, \mathcal{R'}_U)$ such that $X$ is a minimum in $\mathcal{X}(\mathcal{R}_L, \mathcal{R}_U) - \mathcal{X}(\mathcal{R'}_L, \mathcal{R'}_U)$. 

\begin{codebox} 
\Procname{$\proc{DFS}(\mathbf{M}, S, \mathcal{R}_L, \mathcal{R}_U, c)$} 
\li $\mathcal{M} \gets \emptyset \phantom{abcde} \mathcal{G} \gets \mathcal{S} \gets \{ \mathbf{M} \}$
$\phantom{abcdefghij}$  \Comment $\mathcal{S}$ is a linked list
\li \While $\mathcal{S} \ne \emptyset$ \label{li:DFS:while_begin} 
\li \Do  select the head element $\mathbf{Y}$ of $\mathcal{S}$ \label{li:DFS:select_node}
\li \Repeat \label{li:DFS:repeat_begin} 
\li      $\langle \mathbf{X}, \mathbf{Y} \rangle \gets \proc{Select-Unvisited-Adjacent}(\mathbf{Y}, \mathcal{G}, S, \mathcal{R}_L, \mathcal{R}_U)$ \label{li:DFS:select_element}
\li      \If $\mathbf{X} = NIL$
\li          \Then remove $\mathbf{Y}$ from $\mathcal{S}$  \label{li:DFS:pop}
\li          \Else insert $\mathbf{X}$ into $\mathcal{S}$  as the head \label{li:DFS:push}
\li                   insert $\mathbf{X}$ into $\mathcal{G}$
\li                   insert $\mathbf{X}[element]$ into $\mathcal{M}$ \label{li:DFS:store_minimum_2}
\li                 $\langle \mathcal{G}, \mathcal{R}_L, \mathcal{R}_U \rangle \gets \proc{Node-Pruning}(\mathbf{X}, \mathbf{Y}, \mathcal{G}, \mathcal{R}_L, \mathcal{R}_U, c)$ \label{li:DFS:pruning_1}
             % \End  % if X = NIL
          \End   % repeat
\li \Until $\mathbf{X} = NIL$ {\bf or} $c(\mathbf{X}[element]) \le c(\mathbf{Y}[element]$) \label{li:DFS:repeat_end} 
\li \If $\mathbf{Y}[lower\_adjacent] = \emptyset$ {\bf and} $\mathbf{Y}[element]$ is not covered by $\mathcal{R}_L$    \label{li:DFS:pruning_5}
\li      \Then $\langle \mathcal{G}, \mathcal{R}_L \rangle \gets \proc{Lower-Pruning}(\mathbf{Y}, \mathcal{G}, \mathcal{R}_L)$  \label{li:DFS:pruning_5_1}
         \End
\li \If $\mathbf{Y}[upper\_adjacent] = \emptyset$ {\bf and} $\mathbf{Y}[element]$ is not covered by $\mathcal{R}_U$    \label{li:DFS:pruning_6}
\li     \Then $\langle \mathcal{G}, \mathcal{R}_U \rangle \gets \proc{Upper-Pruning}(\mathbf{Y}, \mathcal{G}, \mathcal{R}_U)$  \label{li:DFS:pruning_6_1}
         \End
\li \If $\mathbf{Y}[lower\_adjacent] = \emptyset$ {\bf and} $\mathbf{Y}[upper\_adjacent] = \emptyset$
\li     \Then  remove $\mathbf{Y}$ from $\mathcal{G}$
         \End
\li    $\mathcal{S} \gets \mathcal{S} \cap \mathcal{G}$   $\phantom{abcdefghij}$ \Comment removes nodes that were pruned  \label{li:DFS:clean_stack}
     \End  \label{li:DFS:while_end} 
\li  \For each $\mathbf{X}$ in $\mathcal{G}$
\li         \Do \If $\mathbf{X}[lower\_adjacent] = \emptyset$
\li                     \Then  $\mathcal{R}_L \gets \proc{Update-Lower-Restriction}(\mathbf{X}[element], \mathcal{R}_L)$
                    \End
\li               \If $\mathbf{X}[upper\_adjacent] = \emptyset$
\li                     \Then  $\mathcal{R}_U \gets \proc{Update-Upper-Restriction}(\mathbf{X}[element], \mathcal{R}_U)$
                    \End
            \End
\li \Return $\langle \mathcal{M}, \mathcal{R}_L, \mathcal{R}_U \rangle$
\end{codebox}
\onehalfspacing

\subsubsection{The main algorithm}
\label{sec:UCS:main}

Finally,  we present the main algorithm. 

\paragraph{Subroutine description} 
$\proc{UCS}$ receives a non-empty set $S$, a decomposable in U-shaped curves cost function  $c$, and returns a collection of all elements in $\mathcal{P}(S)$ of minimum cost.

\begin{codebox} 
\Procname{$\proc{UCS}(S, c)$} 
\li $\mathcal{M} \gets \mathcal{R}_L \gets \mathcal{R}_U \gets \emptyset$
\li \Repeat  \label{li:UCS:begin}   
\li     \If $\proc{Select-Direction} = \mbox{\tt UP }$ 
\li         \Then $A \gets \proc{Minimal-Element}(S, \mathcal{R}_L)$ \label{li:UCS:minimal}
\li                    \If $A \ne NIL$ 
\li                          \Then $\mathcal{R}_L \gets \proc{Update-Lower-Restriction}(A, \mathcal{R}_L)$ \label{li:UCS:update_lower} 
\li                                     \If $A$ is not covered by $\mathcal{R}_U$
\li                                            \Then $\mathbf{A}[element] \gets A$
\li                                                      $\mathbf{A}[upper\_adjacent] \gets \mathbf{A}[unverified]  \gets S - A$
\li                                                      $\mathbf{A}[lower\_adjacent] \gets \emptyset$
\li                                                      $\langle \mathcal{N}, \mathcal{R}_L, \mathcal{R}_U \rangle \gets \proc{DFS}(\mathbf{A}, S, \mathcal{R}_L, \mathcal{R}_U, c)$
\li                                                      $\mathcal{M} \gets \mathcal{M} \cup \{ A \} \cup \mathcal{N}$
                                                  \End
                              \End
\li         \Else  $A \gets \proc{Maximal-Element}(S, \mathcal{R}_U)$ \label{li:UCS:maximal}
\li                    \If $A \ne NIL$ 
\li                          \Then $\mathcal{R}_U \gets \proc{Update-Upper-Restriction}(A, \mathcal{R}_U)$ \label{li:UCS:update_upper} 
\li                                     \If $A$ is not covered by $\mathcal{R}_L$
\li                                            \Then $\mathbf{A}[element] \gets A$ 
\li                                                      $\mathbf{A}[lower\_adjacent] \gets \mathbf{A}[unverified] \gets A$
\li                                                      $\mathbf{A}[upper\_adjacent] \gets \emptyset$
\li                                                      $\langle \mathcal{N}, \mathcal{R}_L, \mathcal{R}_U \rangle \gets \proc{DFS}(\mathbf{A}, S, \mathcal{R}_L, \mathcal{R}_U, c)$
\li                                                      $\mathcal{M} \gets \mathcal{M} \cup \{ A \} \cup \mathcal{N}$
                              \End
              \End
        \End
\li \Until $A = NIL$    \label{li:UCS:end}
\li \Return $\{ M \in \mathcal{M} :  c(M) \mbox{ is minimum} \}$ \label{li:UCS:return}
\end{codebox}
\onehalfspacing

The function $\proc{Select-Direction}$ returns either {\tt UP} or {\tt
  DOWN}, according to an arbitrary probability distribution. 

\paragraph{Time complexity analysis}
Let $u$ be an integer proportional to the number of iterations of the
loop in the lines \ref{li:UCS:begin}--\ref{li:UCS:end} until the
condition $A = NIL$ is satisfied. The $\proc{UCS}$ algorithm demands $O(f(n) u +  n^2 u^2)$
computational time ~\citep{Reis:Tese}. Once an execution of this algorithm may explore a fraction of the search space that is proportional to the size of a Boolean lattice of degree $n$, the actual upper bound for the asymptotic computational time of $\proc{UCS}$ is $O(f(n) 2^n + n^2 2^n) = O((f(n) + n^2) 2^n)$.

% Experimental results
%
\section{Experimental Results}
\label{sec:results}

In this section, we present some experimental evaluation of the $\proc{UCS}$ algorithm, including a comparison with other optimal and suboptimal algorithms. To compare $\proc{UCS}$ with an optimal search algorithm, the branch-and-bound algorithm presented by \citet{Narendra:1977} was not a suitable benchmark, since it works with the hypothesis that the cost function is monotonic. Moreover, it solves a different problem, since it explores only subsets of $S$ that have an arbitrary cardinality~\citep{Narendra:1977}. Therefore, to provide a better comparison basis other than the exhaustive search, we employed an optimal branch-and-bound algorithm to tackle the U-curve problem, the $\proc{U-Curve-Branch-and-Bound}$ ($\proc{UBB}$) algorithm, which generalizes the kinds of instances solved by the algorithm proposed by Narendra and Fukunaga. The $\proc{UBB}$ algorithm is presented in~\citet{Reis:Tese}. To compare $\proc{UCS}$ with a suboptimal search algorithm, we used the $\proc{SFFS}$ heuristic.

To carry out the experimental evaluation of the $\proc{UCS}$ algorithm, we also developed {\tt featsel}, an object-oriented framework coded in C++ that allows implementation of solvers and cost functions over a Boolean lattice search space ~\citep{Reis:Tese}. For the experiments showed in this article, we implemented the $\proc{UCS}$ algorithm, the $\proc{SFFS}$ heuristic, and the $\proc{UBB}$ algorithm. The experiments were done using a $32$-bit PC with clock of $2.8$ GHz and memory of $8$ GB.

Each experiment employed either simulated instances or real data from
W-operator design. We will now describe the cost functions that were
used in the experiments.

\paragraph{Simulated instances} In this type of experiment, the
performance of the algorithms were evaluated using ``hard''
instances. For each instance size, it was generated random synthetic U-curve instances
through the methodology presented by~\citet{Reis:Tese}. For each instance size, one hundred instances were produced, with the number of features ranging from seven to eighteen. For each single instance, we carried out on it the three algorithms. Finally, for each size of instance, we executed the $\proc{UCS}$, $\proc{UBB}$ and $\proc{SFFS}$ algorithms, taking the average required time of the execution, the average number of computed nodes (i.e., the number of times the cost function was computed), and counting, for each algorithm, the number of times it found a best solution.

\paragraph{Design of W-operators} A W-operator is an image transformation that is locally defined inside a window $W$ (i.e., a subset of the integer plane) and translation invariant ~\citep{Barrera:2000}. An image transformation is characterized by a classifier. One step in the design of the classifier is the feature selection (i.e., the choice of the subsets of W that optimizes some quality criterion). The minimization of mean conditional entropy is a quality criterion that was showed effective in W-operator window design ~\citep{Martins-Jr:2006}.

Let $X$ be a subset of a window $W$ and $\bf X $ be a random variable in $\mathcal{P}(X)$. The conditional entropy of a binary random variable $Y$ given ${\bf X} = {\bf x }$ is given by 
\begin{align*}
H(Y | {\bf X} = {\bf x}) = - \sum_{y \in Y}{P(Y = y | {\bf X} = {\bf x}) \log{ P(Y = y| {\bf X} = {\bf x}) } },
\end{align*}
in which $P(.)$ is the probability distribution function. By definition, $log 0 = 0$. The mean conditional entropy of $Y$ given $\bf X$ is expressed by
\begin{align*}
E[H(Y | {\bf X})] = \sum_{{\bf x} \in {\bf X}}{ H(Y | {\bf X} = {\bf x}) P({\bf X} = {\bf x})}.
\end{align*}

In practice, the values of the conditional probability distributions are estimated, thus a high error estimation may be caused by the lack of sufficient sample data ({\it i.e.}, rarely observed pairs $\langle y, {\bf x} \rangle$ may be underrepresented). In order to try to circumvent the insufficient number of samples, it is adopted the penalty for pairs of values that have a unique observation: it is considered a uniform distribution, thus leading to the highest entropy. Therefore, adopting the penalization,  the estimation of the mean conditional entropy is given by
\begin{equation} \label{eq:MCE}
\hat{E}[H(Y | {\bf X})] = \frac{N}{t} + \sum_{{\bf x} \in {\bf X} : \hat{P}({\bf x }) > \frac{1}{t}}{ \hat{H}(Y | {\bf X} = {\bf x}) \hat{P}({\bf X} = {\bf x}) },
\end{equation}
in which $N$ is the number of values of $\bf X$ with a single occurrence in the samples and $t$ is the total number of samples. 

The penalized mean conditional entropy (Equation \ref{eq:MCE}) was used in this experiment as the cost function. The training samples were obtained from fourteen pairs of binary images presented in~\citet{Martins-Jr:2006}, running a window of size $16$ over each of observed image and obtaining the corresponding value of the transformed image. The result was fourteen sets with $1 152$ samples each. For each set, we executed the $\proc{UCS}$, $\proc{UBB}$ and $\proc{SFFS}$ algorithms, taking the required time of the execution, the number of computed nodes (i.e., the number of times the cost function was computed), and verified, for each algorithm, if it found a best solution.

\paragraph{Types of experiments} For each type of cost function, we
performed two types of experiments: first, we evaluated the
$\proc{UCS}$ algorithm as an optimal search algorithm; second, we used
the $\proc{UCS}$ algorithm as an heuristic, in which the stop
criterion is to reach a maximum number of times that the cost function
may be computed; such threshold is obtained through pre-processing
steps that will be explained later. The results of simulated data
experiments were also employed to perform an analysis of some aspects
of the $\proc{UCS}$ algorithm dynamics; such analysis will be presented in Discussion section.

\subsection{Optimal-search experiments}

In Table \ref{tab:comparison_subset_sum}, we summarize the results of the optimal search experiment with simulated instances. In the semantic point of view, $\proc{SFFS}$ had a poor performance in this experiment: first, on the smallest instances ($2^7$) it found the best solution only in $44$ out of $100$; second, as the size of the instance increases, it was less likely that $\proc{SFFS}$ provided an optimal solution. $\proc{UBB}$ and $\proc{UCS}$ are equivalent, since both give an optimal solution. In the computational performance point of view, for instances from size $2^7$ to $2^{12}$, $\proc{SFFS}$ and $\proc{UCS}$ computed the cost function a similar number of times. $\proc{UCS}$ always gave an optimal solution with a better ratio between number of size of instance / number of computed nodes than $\proc{UBB}$; for example, for instances of size $18$, the ratios of $\proc{UCS}$ and $\proc{UBB}$ were, respectively, $11.7$ and $1.87$. On the other hand, $\proc{UBB}$ always gave an optimal solution with a better ratio between size of instance / computational time  than $\proc{UCS}$; for example, for instances of size $18$, the ratios of $\proc{UCS}$ and $\proc{UBB}$ were, respectively, $741$ and $17 476$.

\begin{table}[!t] \begin{center} 
\caption{comparison between $\proc{UCS}$, $\proc{SFFS}$ and
  $\proc{UBB}$, using simulated data. Although the
  computational time $\proc{UCS}$ increases faster than the branch and
  bound one, the former demands less computation of the cost function
  (``computed nodes'') than the
  latter.} \label{tab:comparison_subset_sum}
\small
\begin{tabular}{@{}cccccccccccccc@{}} \toprule
\multicolumn{2}{c}{Instance size} && \multicolumn{3}{c}{Time (sec)} && \multicolumn{3}{c}{\# Computed nodes} && \multicolumn{3}{c}{\# The best solution}\\
\cline{1-2} \cline{4-6} \cline{8-10} \cline{12-14}
$|S|$ & $2^{|S|}$  && $\proc{UCS}$ & $\proc{UBB}$ & $\proc{SFFS}$ && $\proc{UCS}$ & $\proc{UBB}$ & $\proc{SFFS}$  && $\proc{UCS}$ & $\proc{UBB}$ & $\proc{SFFS}$\\ \hline
 7 &     128 && 0.03 & 0.01 & 0.02 && 63 & 89 & 114 && 100 & 100 & 44\\ 
 8 &     256 && 0.10 & 0.05 & 0.05 && 104 & 173 & 172 && 100 & 100 & 29\\ 
 9 &     512 && 0.11 & 0.06 & 0.03 && 157 & 328 & 281 && 100 & 100 & 30\\ 
10 &    1 024 && 0.19 & 0.09 & 0.04 && 242 & 679 & 465 && 100 & 100 & 25\\ 
11 &    2 048 && 0.36 & 0.17 & 0.04 && 494 & 1 422 & 469 && 100 & 100 & 15\\ 
12 &    4 096 && 0.69 & 0.27 & 0.06 && 788 & 2 451 & 539 && 100 & 100 & 17\\ 
13 &    8 192 && 1.49 & 0.60 & 0.08 && 1 335 & 5 327 & 822 && 100 & 100 & 10\\ 
14 &   16 384 && 2.90 & 0.95 & 0.08 && 1 922 & 9 374 & 892 && 100 & 100 & 17\\ 
15 &   32 768 && 8.75 & 2.08 & 0.11 && 4 127 & 19 686 & 1 069 && 100 & 100 & 6\\ 
16 &   65 536 && 27.38 & 4.79 & 0.15 && 7 607 & 45 353 & 1 479 && 100 & 100 & 4\\ 
17 &  131 072 && 74.80 & 8.83 & 0.14 && 12 526 & 80 462 & 1 688 && 100 & 100 & 10\\ 
18 &  262 144 && 354.22 & 15.70 & 0.16 && 22 461 & 139 977 & 1 555 && 100 & 100 & 5\\ 
\bottomrule \end{tabular}
\end{center} \end{table}
\onehalfspacing

In Table \ref{tab:comparison_woperator}, we show the results of the optimal search experiment on the design of W-operators, in which each set was used once for the feature selection procedure. Only in one of the W-operator feature selection $\proc{SFFS}$ provided an optimal solution. $\proc{UCS}$ and $\proc{UBB}$ always gave an optimal solution. Moreover, $\proc{UCS}$ gave an optimal solution in less computational time more frequently than $\proc{UBB}$, since the computation of the cost function is much more expensive than in the cost function employed in the simulated data experiments.

\begin{table}[!t] \begin{center} 
\small
\caption{comparison between $\proc{UCS}$, $\proc{SFFS}$ and
  $\proc{UBB}$, using as the cost function a penalized mean
  conditional entropy, during a W-operator design with a window of
  size $16$ (i.e., with a search space of size
  $2^{16}$).} \label{tab:comparison_woperator} 
\bigskip
\begin{tabular}{@{}ccccccccccccc@{}} \toprule
Test  && \multicolumn{3}{c}{Time (sec)} && \multicolumn{3}{c}{\# Computed nodes} && \multicolumn{3}{c}{\# The best solution}\\
\cline{1-1} \cline{3-5} \cline{7-9} \cline{11-13}
&& $\proc{UCS}$ & $\proc{UBB}$ & $\proc{SFFS}$ && $\proc{UCS}$ & $\proc{UBB}$ & $\proc{SFFS}$  && $\proc{UCS}$ & $\proc{UBB}$ & $\proc{SFFS}$\\ \hline
 1 & & 296.2 & 404.4 & 3.2 && 19 023 & 63 779 & 1 325 && 1 & 1 & 0\\ 
 2 & & 146.9 & 425.8 & 4.5 && 9 114 & 64 406 & 1 309 && 1 & 1 & 0\\ 
 3 & & 217.3 & 395.1 & 10.0 && 12 401 & 62 718 & 3 077 && 1 & 1 & 0\\ 
 4 & & 421.5 & 312.6 & 1.8 && 25 622 & 60 265 & 940 && 1 & 1 & 0\\ 
 5 & & 378.0 & 346.5 & 4.3 && 24 690 & 62 484 & 1 680 && 1 & 1 & 0\\ 
 6 & & 139.5 & 442.7 & 17.6 && 7 682 & 65 197 & 3 248 && 1 & 1 & 0\\ 
 7 & & 232.7 & 522.0 & 5.2 && 13 464 & 64 755 & 1 566 && 1 & 1 & 1\\ 
 8 & & 382.8 & 252.5 & 0.9 && 24 553 & 58 652 & 615 && 1 & 1 & 0\\ 
 9 & & 345.6 & 457.0 & 5.2 && 18 461 & 63 461 & 1 736 && 1 & 1 & 0\\ 
10 & & 382.0 & 485.3 & 8.0 && 19 152 & 63 966 & 2 357 && 1 & 1 & 0\\ 
11 & & 398.7 & 354.3 & 3.9 && 22 752 & 61 183 & 1 851 && 1 & 1 & 0\\ 
12 & & 235.8 & 414.4 & 4.5 && 14 017 & 64 395 & 1 720 && 1 & 1 & 0\\ 
13 & & 220.6 & 478.8 & 6.8 && 12 118 & 64 700 & 2 077 && 1 & 1 & 0\\ 
14 & & 179.7 & 483.6 & 12.6 && 9 291 & 65 469 & 2 695 && 1 & 1 & 0\\ 
\cline{1-1} \cline{3-5} \cline{7-9} \cline{11-13}
Average && 284.1 & 412.5 & 6.4 && 16 595 & 63 245 & 1 871 && 1 & 1 & 0.07\\ 
\bottomrule \end{tabular}
 \end{center} \end{table}
\onehalfspacing

\subsection{Suboptimal-search experiments}

The suboptimal-search experiments were performed in three steps: the first two steps were a pre-processing, used to produce a criterion to stop the search; the third step was the actual suboptimal search. In the following, we describe how these steps were executed for each size of instance in the simulated data and for each instance in the real data.

\begin{itemize}
\item first, the $\proc{SFFS}$ heuristic was executed over the instances. For the simulated data, we produced an average of the minimum cost obtained from each of the hundred instances. For the real data, we kept the minimum cost obtained from a given instance;

\item second, $\proc{UCS}$ and $\proc{UBB}$ were executed over the instances, this time using as a stop criterion the threshold obtained in the first step. For the simulated data, we produced, for each algorithm, an average of the number of computed nodes from each of the hundred instances, keeping the greatest value. For the real data, for each algorithm, we took the number of computed nodes, keeping the greatest value;

\item third, $\proc{SFFS}$, $\proc{UCS}$, and $\proc{UBB}$ were executed over the instances again, this time using as a stop criterion the threshold obtained in the second step.
\end{itemize} 

Finally, the results were organized in the same way presented in the tables of the optimal-search experiments. The only difference is that in the column ``\# The best solution'' (i.e., the number of times an algorithm achieved a minimum), the counting might not necessarily involve global minima.

In Table~\ref{tab:comparison_subset_sum_heuristic_search_step}, we
summarize the results of the suboptimal-search experiment over
simulated instances, while in Table we summarize the results of its
third step. As the size of instance is constant in each line of the
table, $\proc{UCS}$ has a better semantic than $\proc{SFFS}$ and
$\proc{UBB}$, and $\proc{UBB}$ is better than $\proc{SFFS}$.

\begin{table}[!t] \begin{center} 
\small
 \caption{pre-processing that generates thresholds for a suboptimal search in simulated instances.} \label{tab:comparison_subset_sum_heuristic_preprocessing_step}
\bigskip
\begin{tabular}{@{}ccccccccc@{}} \toprule
\multicolumn{2}{c}{Size of instance} & \phantom{abc} & \multicolumn{3}{c}{\# Computed nodes} & \phantom{abc} & Threshold\\
\cline{1-2} \cline{4-6} \cline{8-8}
$|S|$ & $2^{|S|}$  && $\proc{UCS}$ & $\proc{UBB}$ & $\proc{SFFS}$ && \\ \hline
 7 &     128 && 25.88 & 41.84 & 133.13 && 134\\ 
 8 &     256 && 30.79 & 72.54 & 178.17 && 179\\ 
 9 &     512 && 47.62 & 134.38 & 276.11 && 277\\ 
10 &    1 024 && 65.97 & 274.75 & 441.95 && 442\\ 
11 &    2 048 && 73.48 & 399.60 & 520.39 && 521\\ 
12 &    4 096 && 105.58 & 814.07 & 588.28 && 815\\ 
13 &    8 192 && 138.89 & 1 760.80 & 1 023.65 && 1761\\ 
14 &   16 384 && 161.51 & 2 191.74 & 805.04 && 2192\\ 
15 &   32 768 && 167.45 & 6 334.63 & 1 234.98 && 6335\\ 
16 &   65 536 && 196.45 & 8 803.55 & 1 396.27 && 8804\\ 
17 &  131 072 && 275.00 & 18 286.13 & 1 805.41 && 18287\\ 
18 &  262 144 && 299.02 & 25 890.91 & 1 795.37 && 25891\\ 
\bottomrule \end{tabular} \end{center} \end{table}
\onehalfspacing

\begin{table}[!t] \begin{center} 
\caption{suboptimal search in simulated instances; the thresholds for
  each test were obtained through a pre-processing whose results are showed in
  Table~\ref{tab:comparison_subset_sum_heuristic_preprocessing_step}.} \label{tab:comparison_subset_sum_heuristic_search_step}
\small
\bigskip
\begin{tabular}{@{}cccccccccccccc@{}} \toprule
\multicolumn{2}{c}{Instance size} & & \multicolumn{3}{c}{Time (sec)} & & \multicolumn{3}{c}{\# Computed nodes} & & \multicolumn{3}{c}{\# The best solution}\\
\cline{1-2} \cline{4-6} \cline{8-10} \cline{12-14}
$|S|$ & $2^{|S|}$  && $\proc{UCS}$ & $\proc{UBB}$ & $\proc{SFFS}$ && $\proc{UCS}$ & $\proc{UBB}$ & $\proc{SFFS}$  && $\proc{UCS}$ & $\proc{UBB}$ & $\proc{SFFS}$\\ \hline
 7 &     128 & & 0.03 & 0.02 & 0.02 && 57 & 87 & 65 && 100 & 100 & 31\\ 
 8 &     256 & & 0.07 & 0.03 & 0.02 && 93 & 134 & 91 && 100 & 75 & 28\\ 
 9 &     512 & & 0.11 & 0.05 & 0.02 && 163 & 230 & 139 && 98 & 62 & 14\\ 
10 &    1 024 & & 0.19 & 0.08 & 0.03 && 271 & 386 & 228 && 99 & 43 & 10\\ 
11 &    2 048 & & 0.29 & 0.08 & 0.04 && 354 & 446 & 256 && 93 & 47 & 10\\ 
12 &    4 096 & & 0.50 & 0.09 & 0.05 && 512 & 655 & 348 && 93 & 47 & 12\\ 
13 &    8 192 & & 0.90 & 0.20 & 0.09 && 981 & 1 518 & 824 && 92 & 42 & 11\\ 
14 &   16 384 & & 1.21 & 0.24 & 0.07 && 1 432 & 1 826 & 731 && 97 & 37 & 9\\ 
15 &   32 768 & & 3.65 & 0.60 & 0.16 && 3 050 & 5 174 & 1 234 && 97 & 45 & 15\\ 
16 &   65 536 & & 6.38 & 0.91 & 0.12 && 4 957 & 7 615 & 1 396 && 97 & 39 & 5\\ 
17 &  131 072 & & 21.02 & 1.83 & 0.21 && 9 834 & 15 825 & 1 805 && 96 & 28 & 4\\ 
18 &  262 144 & & 30.52 & 2.63 & 0.20 && 14 381 & 22 224 & 1 795 && 95 & 31 & 4\\ 
\bottomrule \end{tabular} 
\end{center} \end{table}
\onehalfspacing

In Table \ref{tab:comparison_woperator_heuristic_search_step}, we
summarize the results of the suboptimal-search experiment over W-operator design data. $\proc{UCS}$ found the best solution in all instances, $\proc{UBB}$ found the best solution twice (Tests $7$ and $10$), and $\proc{SFFS}$ found a best solution only once (Tests $7$).  $\proc{UCS}$ found an optimal solution in $65\%$ of the instances, while $\proc{UBB}$ found an optimal solution in $14\%$ of the instances and $\proc{SFFS}$ found an optimal solution in $7\%$ of the instances. Moreover, $\proc{UCS}$ gave an optimal solution in less computational time more frequently than $\proc{UBB}$, since the computation of the cost function is much more expensive than in the cost function employed in the simulated data experiments.

\begin{table}[!t] \begin{center} 
\small
\caption{pre-processing that generates thresholds for a suboptimal search in W-operator design.} \label{tab:comparison_woperator_heuristic_preprocessing_step}
\bigskip
\begin{tabular}{@{}ccccccc@{}} \toprule
Test number & \phantom{abc} & \multicolumn{3}{c}{\# Computed nodes} & \phantom{abc} & Threshold\\
\cline{1-1} \cline{3-5} \cline{7-7}
&& $\proc{UCS}$ & $\proc{UBB}$ & $\proc{SFFS}$ && \\ \hline
 1 && 58 & 26 616 & 1 325 && 26 616\\ 
 2 && 22 & 7 776 & 1 309 && 7 776\\ 
 3 && 6 853 & 46 289 & 3 077 && 46 289\\ 
 4 && 215 & 3 846 & 940 && 3 846\\ 
 5 && 50 & 15 270 & 1 680 && 15 270\\ 
 6 && 143 & 44 655 & 3 248 && 44 655\\ 
 7 && 8 782 & 64 755 & 1 566 && 64 755\\ 
 8 && 175 & 14 019 & 615 && 14 019\\ 
 9 && 287 & 31 517 & 1 736 && 31 517\\ 
10 && 106 & 60 739 & 2 357 && 60 739\\ 
11 && 91 & 7 480 & 1 851 && 7 480\\ 
12 && 98 & 43 591 & 1 720 && 43 591\\ 
13 && 440 & 46 005 & 2 077 && 46 005\\ 
14 && 481 & 43 706 & 2 695 && 43 706\\  
\cline{1-1} \cline{3-5} \cline{7-7}
Average && 1 271.50 & 32 590.29 & 1 871.14 && 32 590.29\\
\bottomrule \end{tabular} \end{center} \end{table}
\onehalfspacing

\begin{table}[!t] \begin{center} 
\caption{suboptimal search in W-operator design; the thresholds for
  each test were obtained through a pre-processing whose results are showed in
  Table~\ref{tab:comparison_woperator_heuristic_preprocessing_step}.} \label{tab:comparison_woperator_heuristic_search_step} 
\small
\bigskip
\begin{tabular}{@{}ccccccccccccc@{}} \toprule
Test & & \multicolumn{3}{c}{Time (sec)} & & \multicolumn{3}{c}{\# Computed nodes} & & \multicolumn{3}{c}{\# The best solution}\\
\cline{1-1} \cline{3-5} \cline{7-9} \cline{11-13}
&& $\proc{UCS}$ & $\proc{UBB}$ & $\proc{SFFS}$ && $\proc{UCS}$ & $\proc{UBB}$ & $\proc{SFFS}$  && $\proc{UCS}$ & $\proc{UBB}$ & $\proc{SFFS}$\\ \hline
 1 & & 307.3 & 132.4 & 3.2 && 18 866 & 26 616 &   1 325 && 1 & 0 & 0\\ 
 2 & & 104.5 & 27.4 & 4.5 && 7 776 & 7 776 & 1 309 && 1 & 0 & 0\\ 
 3 & & 210.6 & 262.7 & 10.0 && 12 299 & 46 289 & 3 077 && 1 & 0 & 0\\ 
 4 & & 20.1 & 10.7 & 1.8 && 3 846 & 3 846 & 940 && 1 & 0 & 0\\ 
 5 & & 155.7 & 61.0 & 4.5 && 15 270 & 15 270 &  1 680 && 1 & 0 & 0\\ 
 6 & & 162.5 & 254.7 & 16.6 && 8 937 & 44 655 & 3 248 && 1 & 0 & 0\\ 
 7 & & 226.7 & 421.4 & 4.9 && 13 304 & 64 755 & 1 566 && 1 & 1 & 1\\ 
 8 & & 81.9 & 46.2 & 0.9 && 14 019 & 14 019 & 615 && 1 & 0 & 0\\ 
 9 & & 332.8 & 187.1 & 5.3 && 18 738 & 31 517 & 1 736 && 1 & 0 & 0\\ 
10 & & 351.2 & 440.1 & 8.1 && 18 996 & 60 739 & 2 357 && 1 & 1 & 0\\ 
11 & & 65.3 & 26.5 & 4.0 && 7 480 & 7 480 & 1851 && 1 & 0 & 0\\ 
12 & & 265.6 & 246.3 & 4.5 && 15 778 & 43 591 & 1 720 && 1 & 0 & 0\\ 
13 & & 240.2 & 295.7 & 6.9 && 11 508 & 46 005 & 2 077 && 1 & 0 & 0\\ 
14 & & 184.3 & 270.0 & 12.7 && 8 928 & 43 706 & 2 695 && 1 & 0 & 0\\ 
\cline{1-1} \cline{3-5} \cline{7-9} \cline{11-13}
Average && 193.5	& 191.6	& 6.3 && 12 553 & 32 590 & 1 871 && 1 & 0.14 & 0.07\\
\bottomrule \end{tabular}
\end{center} \end{table}
\onehalfspacing

\section{Discussion}
\label{sec:discussion}
In this section, we will discuss the theoretical and experimental
results presented in this article, through analyses of some aspects of
the $\proc{UCS}$ algorithm dynamics and performance. We present these
analyses in the following order: first, the $\proc{UCS}$ dynamics,
which is analyzed on simulated data; second, a comparison of
performance between $\proc{UCS}$ and $\proc{UBB}$, which is performed
on the results obtained from the optimal experiments on real data.

To compare the performance between $\proc{UCS}$ and $\proc{UBB}$ in
the simulated experiments, we executed these algorithms on one hundred
different simulated instances of the same size. For each instance, we
stored the required time to compute all calls of the cost function and
the required time of a whole execution of the algorithm. Finally, we
calculated the averages for the hundred instances of the required time
of the cost function and of the required time of a whole execution. In
Figures \ref{fig:ucs_bb_histogram_simulated:A} and
\ref{fig:ucs_bb_histogram_simulated:B}, we show histograms of the
computing time (in seconds) demanded by, respectively, $\proc{UCS}$
and $\proc{UBB}$ on the $12$ simulated experiments showed in Table
\ref{tab:comparison_subset_sum}. The red bars represent time spent on
computation of the cost function, while the green bars are the time
spent in the remaining actions of each algorithm. On the one hand,
$\proc{UCS}$ expends much more computational time in tasks that do not
involve computing the cost function. On the other hand, $\proc{UBB}$
expends approximately the same amount of time for computing the cost
function and to perform other tasks. $\proc{UCS}$ demands more time
for a whole execution than $\proc{UBB}$, but it spends less time
computing the cost function. To evaluate the performance of
$\proc{UCS}$ and $\proc{UBB}$ in the W-operator design experiments, we
executed the algorithms on the $14$ instances of real data and stored,
for each instance and for each algorithm, the required time to compute
all calls of the cost function and the required time of a whole
execution of the algorithm. In Figures \ref{fig:ucs_bb_histogram:A}
and \ref{fig:ucs_bb_histogram:B}, we show histograms of the computing
time (in seconds) demanded by, respectively, $\proc{UCS}$ and
$\proc{UBB}$ on the estimation of $14$ W-operators showed in Table
\ref{tab:comparison_woperator}. The red bars represent time spent on
computation of the cost function, while the green bars are the time
spent in the remaining actions of each algorithm. $\proc{UCS}$ demands
a smaller computational time for the whole feature selection procedure
of each W-operator, but it expends more computational time in tasks
that do not involve computing the cost function.

\begin{figure}
  \footnotesize
  \centering
  \begin{tabular}{cc}
    \subfigure[] {\scalebox{0.6}{\includegraphics{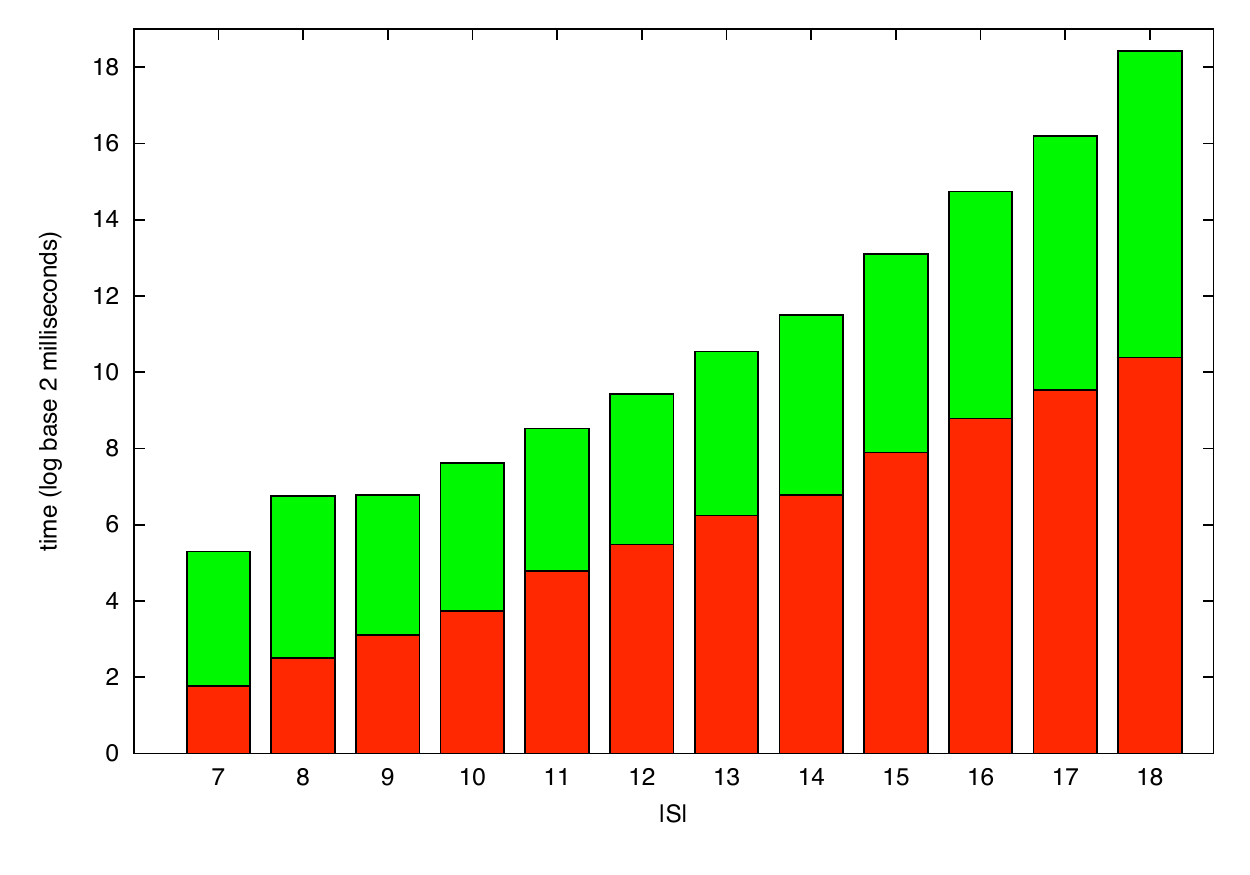}} \label{fig:ucs_bb_histogram_simulated:A}}
    &
    \subfigure[] {\scalebox{0.6}{\includegraphics{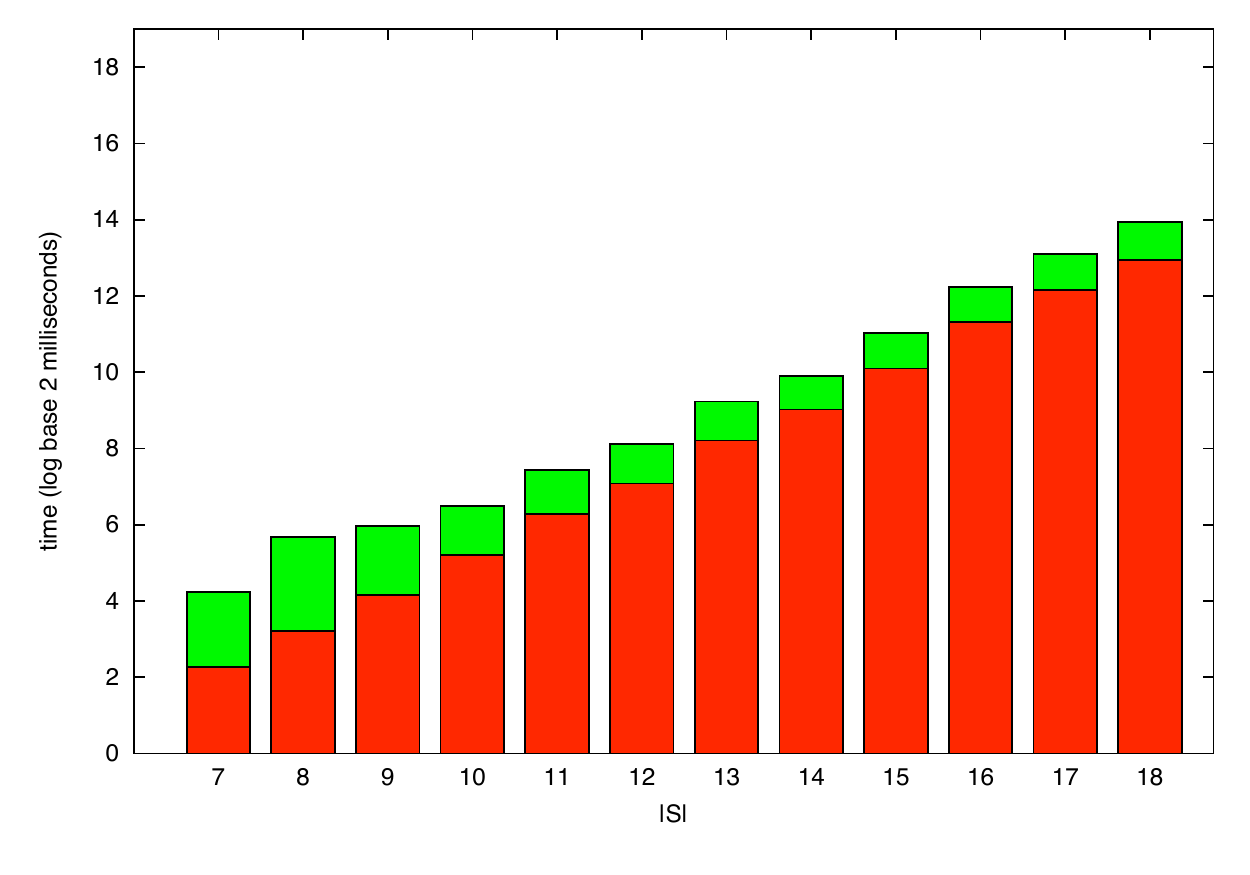}} \label{fig:ucs_bb_histogram_simulated:B}}
    \\
  \end{tabular}
  \caption{comparison of the computing time (in log base 2 seconds) demanded by the $\proc{UCS}$ (\ref{fig:ucs_bb_histogram_simulated:A}) and $\proc{UBB}$ (\ref{fig:ucs_bb_histogram_simulated:B}) on the simulated data showed in Table \ref{tab:comparison_subset_sum}. The red bars represent time spent on computation of the cost function, while the green bars are the time spent in the remaining actions of each algorithm. The $\proc{UCS}$ algorithm always expends more computational time in tasks that does not involve calculating the cost function; on the other hand, the opposite phenomenon occurs with the $\proc{UBB}$.}
  \label{fig:ucs_bb_histogram_simulated} 
\end{figure}

\begin{figure}
  \footnotesize
  \centering
  \begin{tabular}{cc}
    \subfigure[] {\scalebox{0.6}{\includegraphics{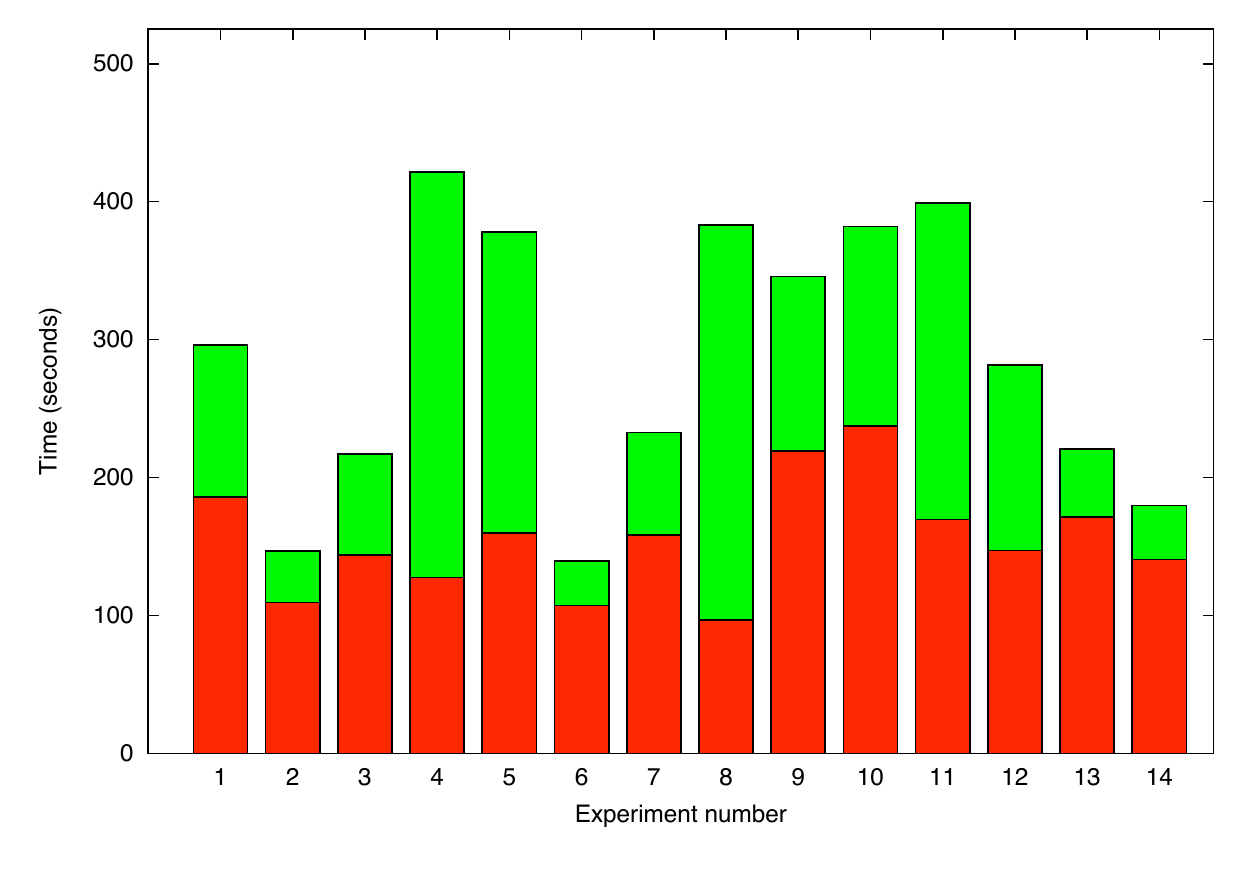}} \label{fig:ucs_bb_histogram:A}}
    &
    \subfigure[] {\scalebox{0.6}{\includegraphics{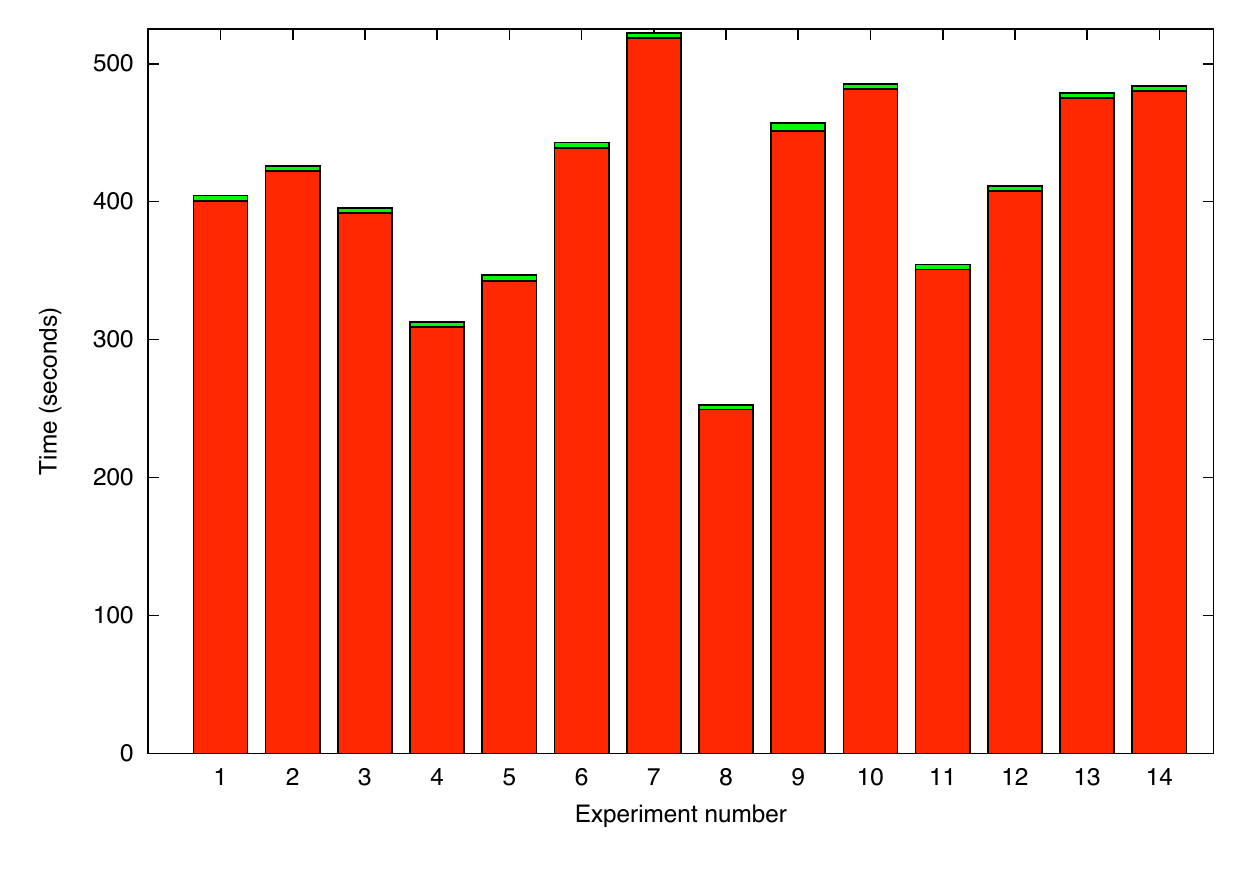}} \label{fig:ucs_bb_histogram:B}}
    \\
  \end{tabular}
  \caption{comparison of the computing time (in seconds) demanded by
    the $\proc{UCS}$ (\ref{fig:ucs_bb_histogram:A}) and $\proc{UBB}$
    (\ref{fig:ucs_bb_histogram:B}) on W-operator design showed in Table \ref{tab:comparison_woperator}. The red bars represent time spent on computation of the cost function, while the green bars are the time spent in the remaining actions of each algorithm. The $\proc{UCS}$ algorithm often demands a smaller computational time for the whole feature selection procedure of each W-operator, but it always expends more computational time that does not involve calculating the cost function.}
  \label{fig:ucs_bb_histogram} 
\end{figure}

To investigate the facts showed in Figures~\ref{fig:ucs_bb_histogram_simulated} and \ref{fig:ucs_bb_histogram}, we did an analysis of the $\proc{UCS}$ algorithm dynamics. We executed the $\proc{UCS}$ on one hundred different instances of the same size. For each instance, we stored the number of times an execution calls the $\proc{DFS}$ subroutine (i.e., a new depth-first search on the current search space) and the number of times an execution calls either $\proc{Minimal-Element}$ or $\proc{Maximal-Element}$ (i.e., returns a new element that might belong to the current search space). Finally, we calculated the averages of calls of $\proc{DFS}$ and calls of $\proc{Minimal-Element}$ or $\proc{Maximal-Element}$ for the hundred instances. In Figures \ref{fig:ucs_DFS_and_minimal:subset_sum:A} and \ref{fig:ucs_DFS_and_minimal:subset_sum:B}, we present graphics of this ratio in function of the experiment label in, respectively, the optimal and suboptimal-search experiments. These graphics show how many times $\proc{UCS}$ calls either $\proc{Minimal-Element}$ or $\proc{Maximal-Element}$ subroutines until it receives a new element of the current search space, which is a necessary condition to call the $\proc{DFS}$ subroutine. As the size of the instance increases, $\proc{UCS}$ needed to perform more calls of the $\proc{Minimal-Element}$ or $\proc{Maximal-Element}$ subroutines until it received a new element of the search space. In both graphics, as the size of instance increases, the distance between the green line (calls of $\proc{Minimal-Element}$ or $\proc{Maximal-Element}$) and the orange line (calls of $\proc{DFS}$) also increases. For instance, in Figure \ref{fig:ucs_DFS_and_minimal:subset_sum:A}, for instances of size $5$, on the average more than $18\%$ of the iterations find a new element of the current search space; on the other hand, for instances of size $18$, on the average less than $2\%$ of the iterations find a new element of current search space. These graphics show that the search for a new element of the search space is a process that contributes to the high percentage of the green bars on the histogram showed in Figure \ref{fig:ucs_bb_histogram:A}.

\begin{figure}
  \footnotesize
  \centering
  \begin{tabular}{cc}
    \subfigure[] {\scalebox{0.28}{\includegraphics{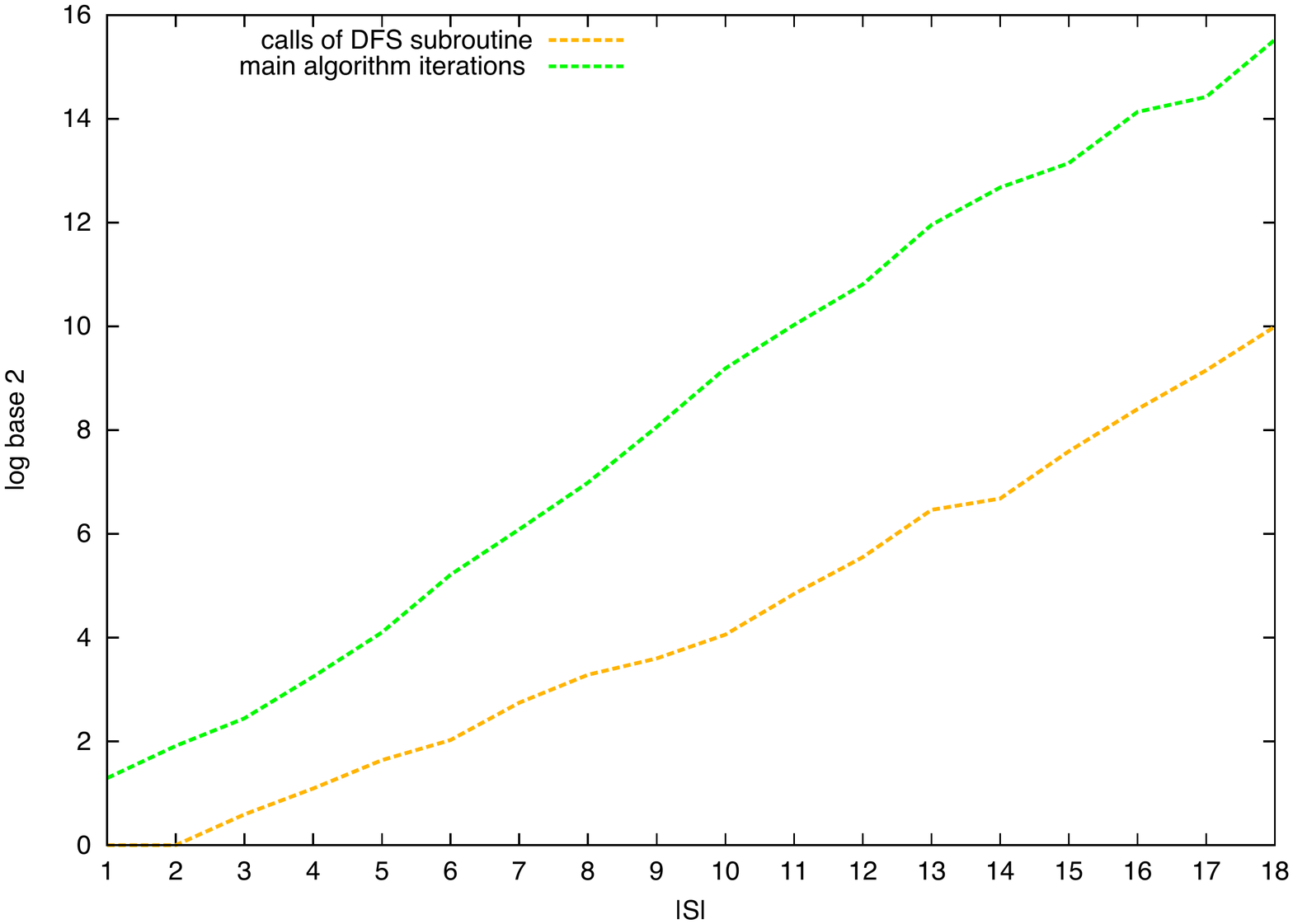}} \label{fig:ucs_DFS_and_minimal:subset_sum:A}}
    &
    \subfigure[] {\scalebox{0.28}{\includegraphics{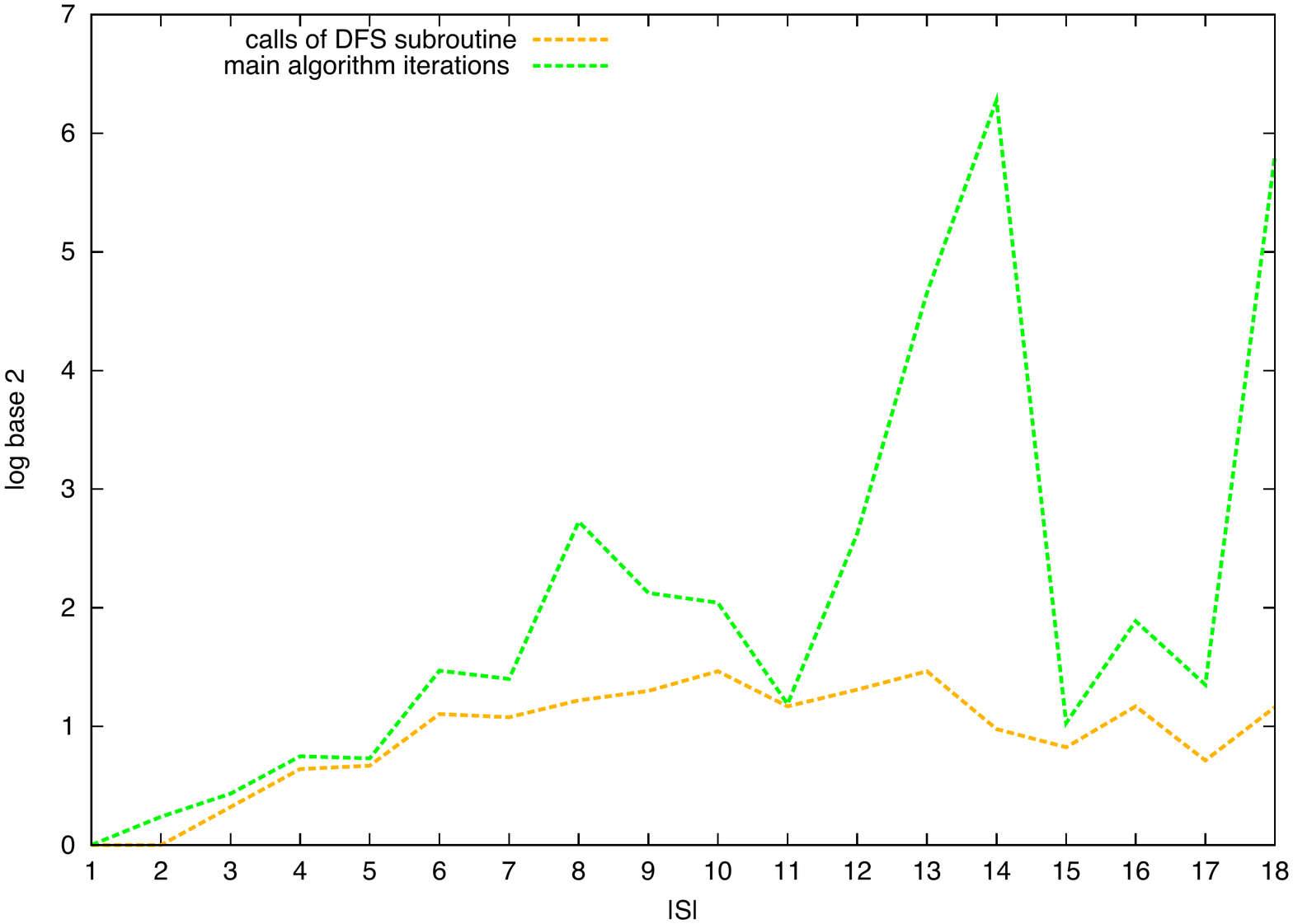}} \label{fig:ucs_DFS_and_minimal:subset_sum:B}}
    \\
  \end{tabular}
  \caption{comparison between the number of times the $\proc{DFS}$ subroutine is called (orange line) and the number of iterations of the main algorithm (green line), during the execution of the $\proc{UCS}$ algorithm on instances of different sizes. Figures \ref{fig:ucs_DFS_and_minimal:subset_sum:A} and \ref{fig:ucs_DFS_and_minimal:subset_sum:B} were produced with the same output used to produce, respectively, Tables \ref{tab:comparison_subset_sum} (i.e., optimal search on simulated data) and \ref{tab:comparison_subset_sum_heuristic_search_step} (i.e., heuristic search on simulated data).}
  \label{fig:ucs_DFS_and_minimal:subset_sum} 
\end{figure}

One property that was empirically observed in Tables \ref{tab:comparison_subset_sum_heuristic_search_step} and \ref{tab:comparison_woperator_heuristic_search_step} is that the velocity of convergence of the $\proc{UCS}$ is much greater than the ones of $\proc{UBB}$ and $\proc{SFFS}$. This property has potential for the design of better optimal-search algorithms, though it has no impact on the complexity analysis of the $\proc{UCS}$ algorithm.

\section{Conclusion}
\label{sec:conclusion}
Aiming to solve the U-curve problem, \citet{Ris:2010} introduced the $\proc{U-Curve}$ algorithm. $\proc{U-Curve}$ is based in the Boolean lattice structure of  the search space. The principles of the algorithm are: (i) if the current search space is not empty, then it pushes in a stack an element $M$; otherwise, it stops; (ii) it selects the top of the stack $M$; if there is an element $N$ adjacent to $M$ such that $c(N) \le c(M)$, then $N$ is pushed in the stack; otherwise, it pops $M$ and uses it to prune the current search space; (iii) if the stack is not empty, then it returns to step (ii); otherwise, it returns to step (i). We demonstrated that this algorithm has a pruning error that leads to suboptimal solutions; the error was pointed out in Figure \ref{fig:ucurve_error}, through a simulation of the algorithm. 

This article introduced a new algorithm, the $\proc{U-Curve-Search}$
($\proc{UCS}$), which is actually optimal to solve the U-curve
problem. $\proc{UCS}$ keeps the general structure of $\proc{U-Curve}$,
but changes the step (ii): it stores in the head of a linked list elements adjacent to $M$ in the current search space, until it either finds an element $N$ such that $c(N) \le c(M)$ (i.e., it reaches the depth-first search criterion) or explores all adjacent elements. Moreover, it uses $M$ to prune the current search space only if it is achieved some sufficient conditions to do so without the risk of losing global minima. $\proc{UCS}$ brings some important improvements: (1) the depth-first search criterion avoids to keep in the memory too many elements, thus providing a better control of the algorithm memory usage; (2) once it achieves a deepest element, it performs less pruning, but very effective ones. The general idea of $\proc{UCS}$ was showed through a simulation presented in Figure \ref{fig:ucurve_corrected}. 

It was performed a diagnosis of the quality of the $\proc{UCS}$
algorithm through experiments on simulated and real data. On the one
hand, $\proc{UCS}$ expends too much computational time looking for a
new element in the search space, which contributes to the
computational time expend in a whole execution of the algorithm. This
fact was showed in Figures \ref{fig:ucs_DFS_and_minimal:subset_sum:A}
and \ref{fig:ucs_DFS_and_minimal:subset_sum:B}. On the other hand, the
$\proc{UCS}$  pruning procedures are effective. This fact was
suggested by Table
\ref{tab:comparison_subset_sum_heuristic_search_step}, which shows
that $\proc{UCS}$ converges relatively fast, and by Figures
\ref{fig:ucs_bb_histogram_simulated:A} and
\ref{fig:ucs_bb_histogram:A}, which  show that $\proc{UCS}$ computes
few nodes while it covers a large subset of the search space. 

A comparison between $\proc{UCS}$, $\proc{UBB}$ and $\proc{SFFS}$ was also made through experiments on simulated and real data. $\proc{UCS}$ had a better performance in the exploration of the search space than $\proc{UBB}$, generally demanding much less computation of the cost function to find an optimal solution. Besides, it often demanded less computational time to find optimal solutions in real data experiments. These facts were showed in Tables \ref{tab:comparison_subset_sum} and \ref{tab:comparison_woperator} and discussed in Section \ref{sec:discussion}.  In the suboptimal-search experiments, $\proc{UCS}$ had a better performance than the other algorithms, finding more frequently a best solution. Moreover, the velocity of convergence of the $\proc{UCS}$ is much greater than the ones of $\proc{UBB}$ and $\proc{SFFS}$. These facts were empirically observed by Tables \ref{tab:comparison_subset_sum_heuristic_search_step} and \ref{tab:comparison_woperator_heuristic_search_step}. $\proc{UBB}$ demanded less computational time than $\proc{UCS}$ in simulated data experiments, in which the computation of the cost function is inexpensive. This fact was showed in Figure \ref{fig:ucs_bb_histogram_simulated}. However, Figures \ref{fig:ucs_DFS_and_minimal:subset_sum:A} and \ref{fig:ucs_DFS_and_minimal:subset_sum:B} suggest that the difference between algorithms is explained by the $\proc{UCS}$ inefficient search for an element of the search space.

Future improvements of the $\proc{UCS}$ algorithm might include:
implementation of a more efficient way to find an element of the
search space; the development of a new DFS, in which it is processed multiple DFS procedures, which might permit a more homogenous exploration of the search space and thus increasing of the pruning efficiency; build of parallelized versions of the algorithm. Finally, another line of future works would be the definition of subspaces of large Boolean lattices, for instance, delimiting the search space by a collection of intervals, or defining parametrized constraints.

\section{Acknowledgments}
This work was supported by CNPq scholarship $142039/2007-1$ and FAPESP project $08/51273-9$ (Marcelo da Silva Reis),  by CNPq grant $306442/2011-6$ and FAPESP Pronex $2011/50761-2$ (Junior Barrera), and by CNPq grant $302736/2010-7$ (Carlos Eduardo Ferreira).

%
% bibliography
%
%\bibliographystyle{model1-num-names} 
% \bibliographystyle{unsrt}

\bibliographystyle{model5-names}  % APA (American Psychological
                                                         % Association) reference style
\bibliography{bib-UCS}

%\renewcommand{\thefigure}{\arabic{figure}}
%\renewcommand{\thetable}{\arabic{table}}

% more space between rows
% \renewcommand{\arraystretch}{1.2}
%\renewcommand{\arraystretch}{1.5}

\end{document}